\UseRawInputEncoding

\documentclass[journal]{IEEEtran}

\usepackage{graphicx}
\usepackage{booktabs}
\usepackage{xcolor}
\usepackage{multirow}
\usepackage{subfigure}

\usepackage{times}  
\usepackage{helvet} 
\usepackage{courier}  
\usepackage[hyphens]{url}  
\usepackage{graphicx} 
\usepackage{setspace}
\usepackage{amsthm}

\usepackage{amssymb}
\usepackage{algorithmicx}
\usepackage[ruled]{algorithm}
\usepackage{algpseudocode}
\usepackage{enumerate}
\usepackage{comment}
\usepackage[tbtags]{amsmath}
\usepackage{xspace}
\usepackage{algorithmicx}
\usepackage[ruled]{algorithm}
\usepackage{algpseudocode}
\usepackage{underscore}
\usepackage{array}
\usepackage{color}
\usepackage{bm}
\usepackage{cite}
\usepackage[pagebackref=true,breaklinks=true,letterpaper=true,colorlinks,bookmarks=false]{hyperref}

\newcommand{\calL}{\mathcal{L}}
\newcommand{\Tab}{Table\xspace}

\newtheorem{proposition}{Proposition}

\ifCLASSINFOpdf
\else
\fi
\begin{document}
	%
	\title{Enhancing Geometric Factors in Model Learning and Inference for Object Detection and Instance Segmentation}
	%
	%
	%
	
	\author{Zhaohui~Zheng,
		Ping~Wang,
		Dongwei~Ren,
		Wei~Liu,
		Rongguang~Ye,
		Qinghua~Hu,
		Wangmeng~Zuo
		\thanks{This work was supported by 
			 National Natural Science Foundation of China under Grants (Nos. 61801326 and U19A2073).}
		
		\thanks{Z. Zheng, P. Wang and R. Ye are with the School of Mathematics, Tianjin University, Tianjin, 300350, China.
			(Email: zh_zheng@tju.edu.cn, wang_ping@tju.edu.cn, ementon@tju.edu.cn)}
		\thanks{D. Ren and W. Zuo are with the School of Computer Science and Technology, Harbin Institute of Technology, Harbin, 150001, China. (Email: rendongweihit@gmail.com, cswmzuo@gmail.com)}
		\thanks{W. Liu and Q. Hu are with the Tianjin Key Laboratory of Machine Learning, College of Intelligence and Computing, Tianjin University, Tianjin, 300350, China. (Email: lewiswestbrook95@gmail.com, huqinghua@tju.edu.cn)}
		\thanks{Corresponding author: Dongwei Ren}	
}%
	
	%
	%

	\markboth{IEEE Transactions on Cybernetics,~Vol.~xx, No.~xx, Month Year}%
	{Shell \MakeLowercase{\textit{et al.}}: Bare Demo of IEEEtran.cls for IEEE Journals}
	%



	\maketitle
	
	\begin{abstract}
	Deep learning-based object detection and instance segmentation have achieved unprecedented progress.
	In this paper, we propose Complete-IoU (CIoU) loss and Cluster-NMS for enhancing geometric factors in both bounding box regression and Non-Maximum Suppression (NMS), leading to notable gains of average precision (AP) and average recall (AR), without the sacrifice of inference efficiency.
	In particular, we consider three geometric factors, i.e., overlap area, normalized central point distance and aspect ratio, which are crucial for measuring bounding box regression in object detection and instance segmentation.
	The three geometric factors are then incorporated into CIoU loss for better distinguishing difficult regression cases.
	The training of deep models using CIoU loss results in consistent AP and AR improvements in comparison to widely adopted $\ell_n$-norm loss and IoU-based loss.
	Furthermore, we propose Cluster-NMS, where NMS during inference is done by implicitly clustering detected boxes and usually requires less iterations.
	Cluster-NMS is very efficient due to its pure GPU implementation, and geometric factors can be incorporated to improve both AP and AR.
	In the experiments, CIoU loss and Cluster-NMS have been applied to state-of-the-art instance segmentation (e.g., YOLACT and BlendMask-RT), and object detection (e.g., YOLO v3, SSD and Faster R-CNN) models.
	Taking YOLACT on MS COCO as an example, our method achieves performance gains as +1.7 AP and +6.2 AR$_{100}$ for object detection, and +1.1 AP and +3.5 AR$_{100}$ for instance segmentation, with 27.1 FPS on one NVIDIA GTX 1080Ti GPU.
	All the source code and trained models are available at \url{https://github.com/Zzh-tju/CIoU}.
\end{abstract}


%
\IEEEpeerreviewmaketitle

\section{Introduction}

%
%
%
%

\IEEEPARstart{O}{bject} detection and instance segmentation have received overwhelming research attention due to their practical applications in video surveillance, visual tracking, face detection and inverse synthetic aperture radar detection \cite{VideoSurveillance,SiamRCNN,human,face,ISAR,Survey}.
Since Deformable Part Model \cite{DPM}, bounding box regression has been widely adopted for localization in object detection.
Driven by the success of deep learning, prosperous deep models based on bounding box regression have been studied, including one-stage \cite{yolov2,yolov3,SSD,DSSD,focalloss,FCOS}, two-stage \cite{fastrcnn,fasterrcnn}, and multi-stage detectors \cite{iterativebboxreg,cascadercnn}.
Instance segmentation is a more challenging task \cite{TensorMask,polarmask,yolact}, where instance mask is further required for accurate segmentation of individuals.
Recent state-of-the-art methods suggest to add an instance mask branch to existing object detection models, e.g., Mask R-CNN \cite{maskrcnn} based on Faster R-CNN \cite{fasterrcnn} and YOLACT \cite{yolact} based on RetinaNet \cite{focalloss}.
In object detection and instance segmentation, dense boxes are usually regressed \cite{SSD,fasterrcnn,maskrcnn,yolact}.
As shown in Fig. \ref{fig:Diversity}, existing loss functions are limited in distinguishing difficult regression cases during training, and it takes tremendous cost to suppress redundant boxes during inference.
In this paper, we suggest to handle this issue by enhancing geometric factors of bounding box regression into the learning and inference of deep models for object detection and instance segmentation.

\begin{figure}[!t]
	\centering
	\includegraphics[width=0.75\columnwidth]{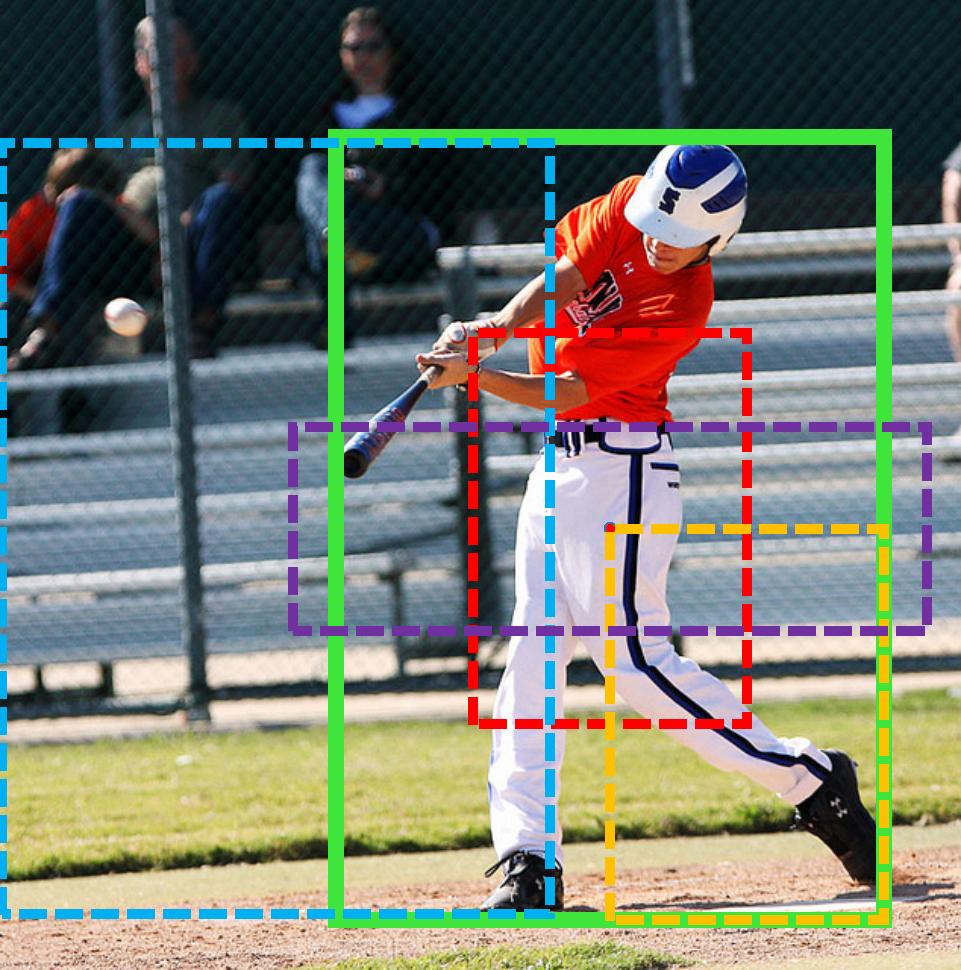}
	\caption{Diversity of bounding box regression, where green box is the ground-truth box.
		First, albeit different ways of overlaps, these regression cases have the same $\ell_1$ loss and IoU loss.
		We propose CIoU loss by considering three geometric factors to distinguish them.
		Second, albeit NMS is a simple post-processing step, it is the bottleneck for suppressing redundant boxes in terms of both accuracy and inference efficiency.
		We then propose Cluster-NMS incorporating with geometric factors for improving AP and AR while maintaining high inference efficiency.
	}
	\label{fig:Diversity}
\end{figure}

In training phase, a bounding box $\mathcal{B} = [x,y,w,h]^{\text{T}}$ is forced to approach its ground-truth box $\mathcal{B}^{gt} = [x^{gt},y^{gt},w^{gt},h^{gt}]^{\text{T}}$ by minimizing loss function $\mathcal{L}$,
\begin{equation}
\underset{\Theta}{\min}\sum_{\mathcal{B}^{gt}\in \mathbb{B}^{gt}}\mathcal{L}(\mathcal{B}, \mathcal{B}^{gt}|\Theta),
\end{equation}
where $\mathbb{B}^{gt}$ is the set of ground-truth boxes, and $\Theta$ is the parameter of deep model for regression.
A typical form of $\mathcal{L}$ is $\ell_n$-norm, e.g., Mean-Square Error (MSE) loss and Smooth-$\ell_1$ loss \cite{huber1964robust}, which have been widely adopted in object detection \cite{RDAD,Fisher}, pedestrian detection \cite{SDSRCNN,SSA-CNN}, text detection \cite{RRPN,RRD,qin2019towards}, 3D detection \cite{voxelnet,pointrcnn}, pose estimation \cite{sun2019deep,iskakov2019learnable}, and instance segmentation \cite{TensorMask,yolact}.
However, recent works suggest that $\ell_n$-norm based loss functions are not consistent with the evaluation metric, i.e., Interaction over Union (IoU), and instead propose IoU-based loss functions \cite{unitbox,boundediouloss,giou}.
For training state-of-the-art object detection models, e.g., YOLO v3 and Faster R-CNN, Generalized IoU (GIoU) loss achieves better precision than $\ell_n$-norm based losses.
However, GIoU loss only tries to maximize overlap area of two boxes, and still performs limited due to only considering overlap areas (refer to the simulation experiments in {Sec. \ref{sec:simulation}}).
As shown in Fig. \ref{fig:regression steps}, GIoU loss tends to increase the size of predicted box, while the predicted box moves towards the target box very slowly.
Consequently, GIoU loss empirically needs more iterations to converge, especially for bounding boxes at horizontal and vertical orientations (see Fig. \ref{fig:finalerror}).

In testing phase, the inference of deep model is often efficient to predict dense boxes, which are left to Non-Maximum Suppression (NMS) for suppressing redundant boxes.
NMS is an essential post-processing step in many detectors \cite{rcnn,SSD,yolov1,focalloss,yolov2,yolov3,yolact,fasterrcnn,fastrcnn}.
In original NMS, a box is suppressed only if it has overlap exceeding a threshold with the box having the highest classification score, which is likely to be not friendly to occlusion cases.
Other NMS improvements, e.g., Soft-NMS \cite{softnms} and Weighted-NMS \cite{CAD}, can contribute to better detection precision.
However, these improved NMS methods are time-consuming, severely limiting their real-time inference.
Some accelerated NMS methods \cite{oro2016work,yolact} have been developed for real-time inference, e.g., Fast NMS \cite{yolact}.
Unfortunately, Fast NMS yields performance drop due to that many boxes are likely to be over-suppressed.

In this paper, we propose to enhance geometric factors in both training and testing phases, where Complete-IoU (CIoU) loss aims to better distinguish difficult regression cases and Cluster-NMS can improve AP and AR without the sacrifice of inference time.
As for CIoU loss, three geometric factors, i.e., overlap area, normalized central point distance and aspect ratio, are formulated as invariant to regression scale.
Benefiting from complete geometric factors, CIoU loss can be deployed to improve average precision (AP) and average recall (AR) when training deep models in object detection and instance segmentation.
From Fig. \ref{fig:regression steps}, CIoU loss converges much faster than GIoU loss, and the incorporation of geometric factors leads to much better match of two boxes. 	


\begin{figure}[!t]
	\centering
	\includegraphics[width=0.8\columnwidth]{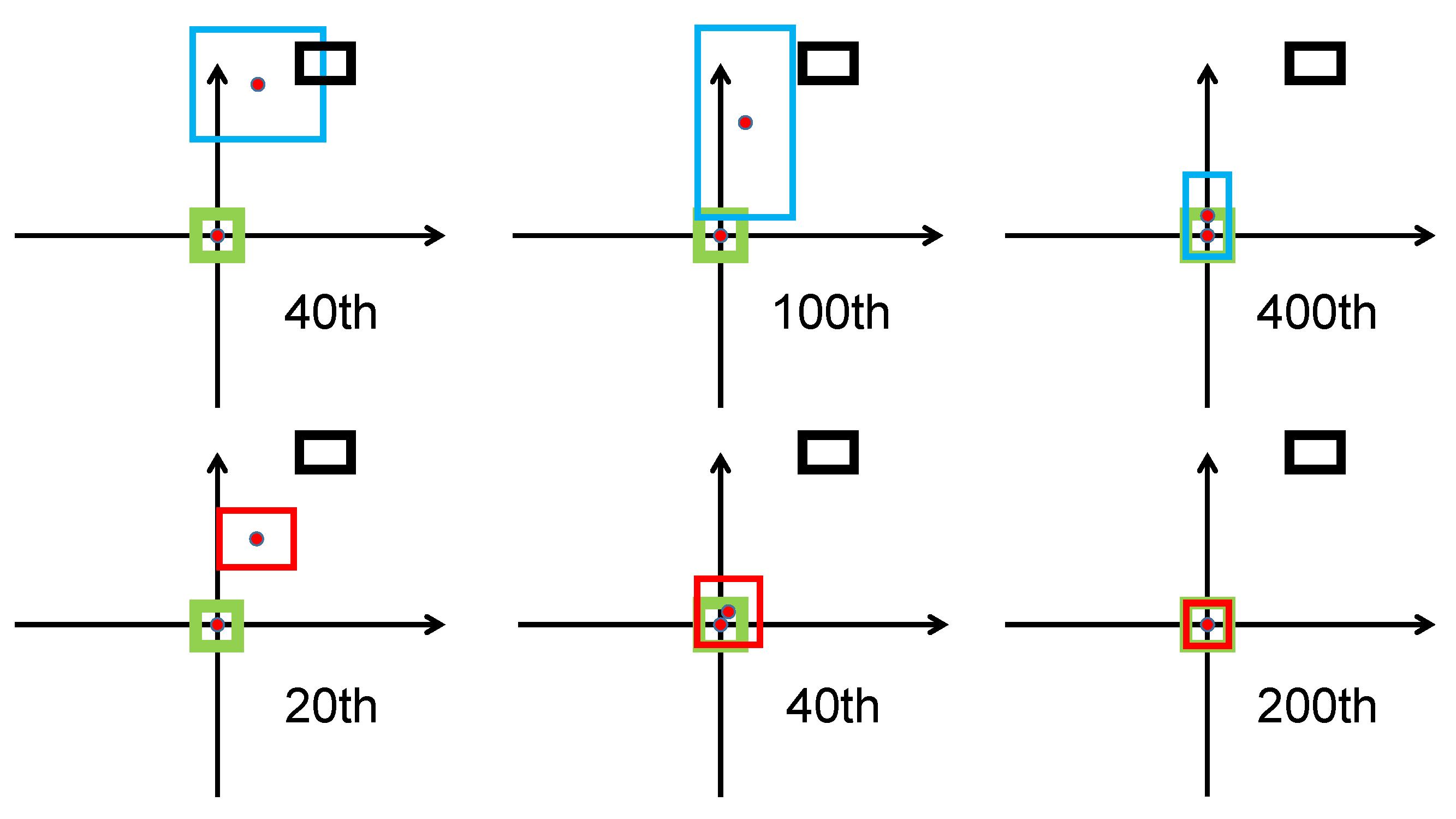}
	\caption{Updating of predicted boxes after different iterations optimized by GIoU loss (first row) and CIoU loss (second row).
		Green and black denote target box and anchor box, respectively.
		Blue and red denote predicted boxes for GIoU loss and CIoU loss, respectively.
		GIoU loss only considers overlap area, and tends to increase the GIoU by enlarging the size of predicted box.
		Benefiting from all the three geometric factors, the minimization of normalized central point distance in CIoU loss gives rise to fast convergence and the consistency of overlap area and aspect ratio contributes to better match of two boxes. }
	\label{fig:regression steps}
\end{figure}

%



We further propose Cluster-NMS, by which NMS can be done by implicitly clustering detected boxes and geometric factors can be easily incorporated, while maintaining high inference efficiency.
First, in Cluster-NMS, redundant detected boxes can be suppressed by grouping them implicitly into clusters.
Cluster-NMS usually requires less iterations, and its suppression operations can be purely implemented on GPU, benefiting from parallel acceleration.
Cluster-NMS can guarantee exactly the same result with original NMS, while it is very efficient.
Then, geometric factors, such as overlap-based score penalty, overlap-based weighted coordinates and normalized central point distance, can be easily assembled into Cluster-NMS.
Benefiting from geometric factors, Cluster-NMS achieves significant gains in both AP and AR, while maintaining high inference efficiency.

In the experiments, CIoU loss and Cluster-NMS have been applied to  several state-of-the-art instance segmentation (e.g., YOLACT \cite{yolact} and BlendMask-RT \cite{blendmask}) and object detection (e.g., YOLO v3 \cite{yolov3}, SSD \cite{SSD} and Faster R-CNN \cite{fasterrcnn}) models.
Experimental results demonstrate that CIoU loss can lead to consistent gains in AP and AR against $\ell_n$-norm based and IoU-based losses for object detection and instance segmentation.
Cluster-NMS contributes to notable gains in AP and AR, and guarantees real-time inference.

%

This paper is a substantial extension of our pioneer work \cite{diou}, comparing with which we have three main changes.
First, the new CIoU loss in this work is a hybrid version of DIoU and CIoU losses in \cite{diou}, and is given more analysis.
Second, a novel Cluster-NMS is proposed to accommodate kinds of NMS methods with high inference efficiency, and DIoU-NMS \cite{diou} can also be easily incorporated to boost their performance.
Third, besides object detection, CIoU loss and Cluster-NMS are further applied to state-of-the-art instance segmentation models, e.g., YOLACT \cite{yolact} and BlendMask-RT \cite{blendmask}.
We summarize the contributions from three aspects:
\begin{itemize}	
	\item A Complete IoU loss, i.e., CIoU loss, is proposed by taking three geometric factors, i.e., overlap area, normalized central point distance and aspect ratio, into account, and results in consistent performance gains for training deep models of bounding box regression.
	
	\item We propose Cluster-NMS, in which geometric factors can be further exploited for improving AP and AR while maintaining high inference efficiency.
	
	\item CIoU loss and Cluster-NMS have been applied to state-of-the-art instance segmentation (e.g., YOLACT and BlendMask-RT) and object detection (e.g., YOLO v3, SSD and Faster R-CNN) models.
	Experimental results validate the effectiveness and efficiency of our methods.
\end{itemize}

The remainder is organized as follows: Sec. \ref{sec:related} briefly reviews related works, Sec. \ref{sec:ciou} proposes CIoU loss by taking complete geometric factors into account, Sec. \ref{sec:Cluster-NMS} presents Cluster-NMS along with its variants by incorporating geometric factors, Sec. \ref{sec:experiment} gives experimental results and Sec. \ref{sec:conclusion} ends this paper with concluding remarks.

\section{Related Work}\label{sec:related}
\subsection{Object Detection and Instance Segmentation}
For a long time bounding box regression has been adopted as an essential component in many representative object detection frameworks \cite{DPM}.
In deep models for object detection, R-CNN series \cite{fasterrcnn,maskrcnn,cascadercnn} adopt two or three bounding box regression modules to obtain higher location accuracy, while YOLO series \cite{yolov1,yolov2,yolov3} and SSD series \cite{SSD,DSSD,STDN} adopt one for faster inference speed.
Recently, in RepPoints \cite{RepPoints}, a rectangular box is formed by predicting several points.
FCOS \cite{FCOS} locates an object by predicting the distances from the sampling points to the top, bottom, left and right sides of the ground-truth box.
PolarMask \cite{polarmask} predicts the length of $n$ rays from the sampling point to the edge of the object in $n$ directions to segment an instance.
There are other detectors such as RRPN \cite{RRPN} and R$^2$CNN \cite{r2cnn} adding rotation angle regression to detect arbitrary orientated objects for remote sensing detection and scene text detection.
%
%
For instance segmentation, Mask R-CNN \cite{maskrcnn} adds an extra instance mask branch on Faster R-CNN, while the recent state-of-the-art YOLACT \cite{yolact} does the same thing on RetinaNet \cite{focalloss}.
%
%
%
%
To sum up, bounding box regression is one key component of state-of-the-art deep models for object detection and instance segmentation.

\subsection{Loss Function for Bounding Box Regression}
Albeit the architectures of deep models have been well studied, loss function for bounding box regression also plays a critical role in object detection.
While $\ell_n$-norm loss functions are usually adopted in bounding box regression, they are sensitive to varying scales.
In YOLO v1 \cite{yolov1}, square roots for $w$ and $h$ are adopted to mitigate this effect, while YOLO v3 \cite{yolov3} uses $2-wh$.
In Fast R-CNN, Huber loss is adopted to obtain more robust training.
Meyer \cite{meyer2019alternative} suggested to connect Huber loss with the KL divergence of Laplace distributions, and further proposed a new loss function to eliminate the transition points between $\ell_1$-norm and $\ell_2$-norm in the Huber loss.
%
%
Libra R-CNN \cite{librarcnn} studies the imbalance issues and proposes Balanced-$\ell_1$ loss.
In GHM \cite{GHM}, the authors proposed a gradient harmonizing mechanism for bounding box regression loss that rectifies the gradient contributions of samples.
IoU loss is also used since Unitbox \cite{unitbox}, which is invariant to the scale.
To ameliorate the training stability, Bounded-IoU loss \cite{boundediouloss} introduces the upper bound of IoU.
GIoU \cite{giou} loss is proposed to tackle the issues of gradient vanishing for non-overlapping cases, but still suffers from the problems of slow convergence and inaccurate regression.
Nonetheless, geometric factors of bounding box regression are actually not fully exploited in existing loss functions. Therefore, we propose CIoU loss by taking three geometric factors into account for better training deep models of object detection and instance segmentation.

\subsection{Non-Maximum Suppression}
NMS is a simple post-processing step in the pipelines of object detection and instance segmentation, but it is the key bottleneck for detection accuracy and inference efficiency.
As for improving detection accuracy, Soft-NMS \cite{softnms} penalizes the detection score of neighbors by a continuous function w.r.t. IoU, yielding softer and more robust suppression than original NMS.
IoU-Net \cite{iounet} introduces a new network branch to predict the localization confidence to guide NMS.
Weighted-NMS \cite{CAD} outputs weighted combination of the cluster based on their scores and IoU.
Recently, Adaptive NMS \cite{adaptivenms} and Softer-NMS \cite{softernms} are proposed to respectively study proper threshold and weighted average strategies.
As for improving inference efficiency, boolean matrix \cite{oro2016work} is adopted to represent IoU relationship of detected boxes, for facilitating GPU acceleration.
A CUDA implementation of original NMS by Faster R-CNN \cite{fasterrcnn} uses logic operations to check the boolean matrix line by line.
Recently, Fast NMS \cite{yolact} is proposed to improve inference efficiency, but it inevitably brings a drop of performance due to the over-suppression of boxes.
In this work, we propose efficient Cluster-NMS, and geometric factors can be readily exploited to obtain significant improvements in both precision and recall.

\section{Complete-IoU Loss} \label{sec:ciou}

For training deep models in object detection, IoU-based losses are suggested to be more consistent with IoU metric than $\ell_n$-norm losses \cite{unitbox,giou,boundediouloss}.
The original IoU loss can be formulated as \cite{giou}, 	%
\begin{equation}\label{eq:iou loss}
\mathcal{L}_{IoU}=1-IoU.
\end{equation}
However, it fails in distinguishing the cases that two boxes do not overlap.
Then, GIoU \cite{giou} loss is proposed,
\begin{equation}\label{eq:giou loss}
\mathcal{L}_{GIoU}=1- IoU + \frac{|\mathcal{C}-\mathcal{B}\cup \mathcal{B}^{gt}|}{|\mathcal{C}|},
\end{equation}
where $\mathcal{C}$ is the smallest box covering $\mathcal{B}$ and $\mathcal{B}^{gt}$, and $|\mathcal{C}|$ is the area of box $\mathcal{C}$.
Due to the introduction of penalty term in GIoU loss, the predicted box will move towards the target box in non-overlapping cases.
GIoU loss has been applied to train state-of-the-art object detectors, e.g., YOLO v3 and Faster R-CNN, and achieves better precision than MSE loss and IoU loss.

\subsection{Analysis to IoU and GIoU Losses} \label{sec:simulation}

To begin with, we analyze the limitations of original IoU loss and GIoU loss.
However, it is very difficult to analyze the procedure of bounding box regression simply from the detection results, where the regression cases in uncontrolled benchmarks are often not comprehensive, e.g., different distances, different scales and different aspect ratios.
Instead, we suggest conducting simulation experiments, where the regression cases should be comprehensively considered, and then the issues of a given loss function can be easily analyzed.
\begin{algorithm}[!tb]	
	
	\footnotesize
	\caption{{Simulation Experiment}}
	\label{algo:RSValue}
	\begin{algorithmic}[1]
		\small{	
			\Require{Loss $\mathcal{L}$ is a continuous bounded function defined on $\mathbb{R}^4_{+}$.
				\newline $\mathbb{B}=\{\{\mathcal{B}_{n,s}\}_{s=1}^{S}\}_{n=1}^{N}$ is the set of anchor boxes at $N=5,000$ uniformly scattered points within the circular region with center $(10,10)$ and radius $3$, and $S=7\times 7$ covers $7$ scales and $7$ aspect ratios of anchor boxes.
				\newline $\mathbb{B}^{gt}=\{\mathcal{B}_{i}^{gt}\}_{i=1}^{7}$ is the set of target boxes that are fixed at $(10,10)$ with area 1, and have $7$ aspect ratios.
			}
			\Ensure{Regression error $\bm{E}\in \mathbb{R}^{T\times N}$}}
		\State Initialize $\bm{E}=\mathbf{0}$ and maximum iteration $T$.
		\State{Do bounding box regression:}
		\For {$n=1$ to $N$}
		\For {$s=1$ to $S$}
		\For {$i=1$ to $7$}
		\For {$t=1$ to $T$}
		\State{$\eta = \begin{cases}0.1 &\text{if}\quad t<=0.8T \\0.01 &\text{if}\quad 0.8T<t<=0.9T \\0.001 &\text{if}\quad t>0.9T\end{cases}$}
		\State{\!\!\!\! $\nabla \mathcal{B}_{n,s}^{t-1}$ is gradient of $\mathcal{L}(\mathcal{B}_{n,s}^{t-1},\mathcal{B}^{gt}_i)$ w.r.t. $\mathcal{B}_{n,s}^{t-1}$}
		\State{$\mathcal{B}_{n,s}^{t}=\mathcal{B}_{n,s}^{t-1}+\eta(2-IoU_{n,s}^{t-1})\nabla \mathcal{B}_{n,s}^{t-1}$}
		\State{$\bm{E}(t,n)=\bm{E}(t,n) +|\mathcal{B}_{n,s}^{t}-\mathcal{B}_i^{gt}|$}
		\EndFor
		\EndFor
		
		\EndFor
		\EndFor
		\State{\Return{$\bm{E}$}}
	\end{algorithmic}
\end{algorithm}

\subsubsection{Simulation Experiment}
In the simulation experiments, we try to cover most of the relationships between bounding boxes by considering geometric factors including distance, scale and aspect ratio, as shown in Fig. \ref{fig:1715ksampling}(a).
In particular, we choose 7 unit boxes (i.e., the area of each box is 1) with different aspect ratios (i.e., 1:4, 1:3, 1:2, 1:1, 2:1, 3:1 and 4:1) as target boxes.
Without loss of generality, the central points of the 7 target boxes are fixed at $(10,10)$.
The anchor boxes are uniformly scattered at 5,000 points.
(\textbf{\emph{i}}) Distance: In the circular region centered at $(10,10)$ with radius 3, 5,000 points are uniformly chosen to place anchor boxes with 7 scales and 7 aspect ratios.
In these cases, overlapping and non-overlapping boxes are included.
(\textbf{\emph{ii}}) Scale: For each point, the areas of anchor boxes are set as $0.5$, $0.67$, $0.75$, $1$, $1.33$, $1.5$ and $2$.
(\textbf{\emph{iii}}) Aspect ratio: For a given point and scale, 7 aspect ratios are adopted, i.e., following the same setting with target boxes (i.e., 1:4, 1:3, 1:2, 1:1, 2:1, 3:1 and 4:1).
All the $5,000 \times 7 \times 7 $ anchor boxes should be fitted to each target box.
To sum up, there are totally $1,715,000 = 7 \times 7 \times 7 \times 5,000$ regression cases.

Then given a loss function $\mathcal{L}$, we can simulate the procedure of bounding box regression for each case using stochastic gradient descent algorithm.
For predicted box $\mathcal{B}_i$, the current prediction can be obtained by
\begin{equation}
\begin{aligned}
\mathcal{B}_{i}^{t}=\mathcal{B}_{i}^{t-1}+\eta(2-IoU_{i}^{t-1})\nabla \mathcal{B}_{i}^{t-1},
\end{aligned}
\end{equation}
where $\mathcal{B}_i^t$ is the predicted box at iteration $t$, $\nabla \mathcal{B}_i^{t-1}$ denotes the gradient of loss $\mathcal{L}$ w.r.t. $\mathcal{B}_i^{t-1}$ at iteration $t-1$, and $\eta$ is the learning rate.
It is worth noting that in our implementation, the gradient is multiplied by $2-IoU_i^{t-1}$ to accelerate the convergence.
The performance of bounding box regression is evaluated using $\ell_1$-norm.
%
%
For each loss function, the simulation experiment is terminated when reaching iteration $T=200$, and the error curves are shown in Fig. \ref{fig:1715ksampling}(b).

\begin{figure}[!t]
	\footnotesize
	\centering
	\includegraphics[width=0.45\columnwidth]{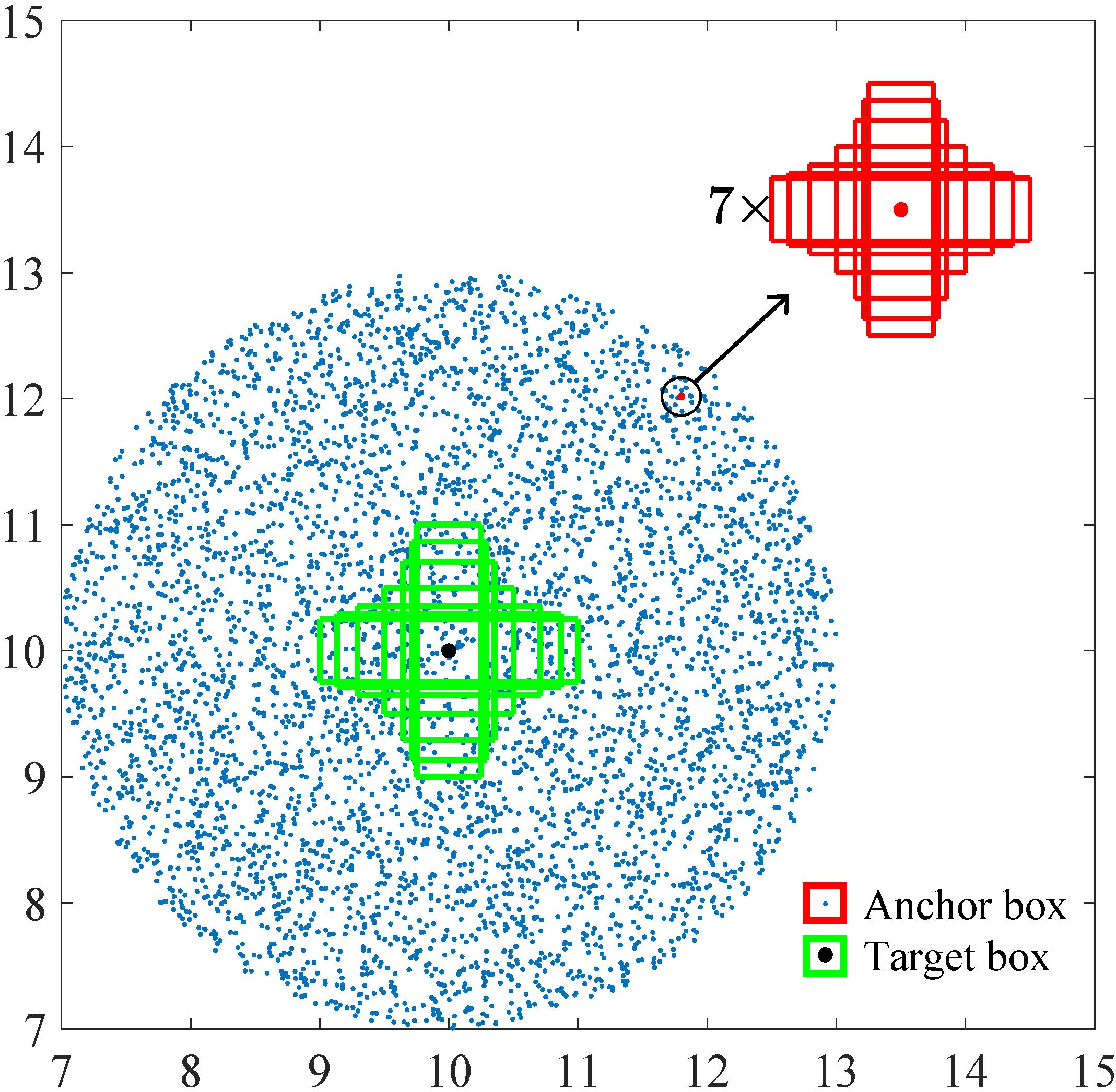}
	\includegraphics[width=0.5\columnwidth]{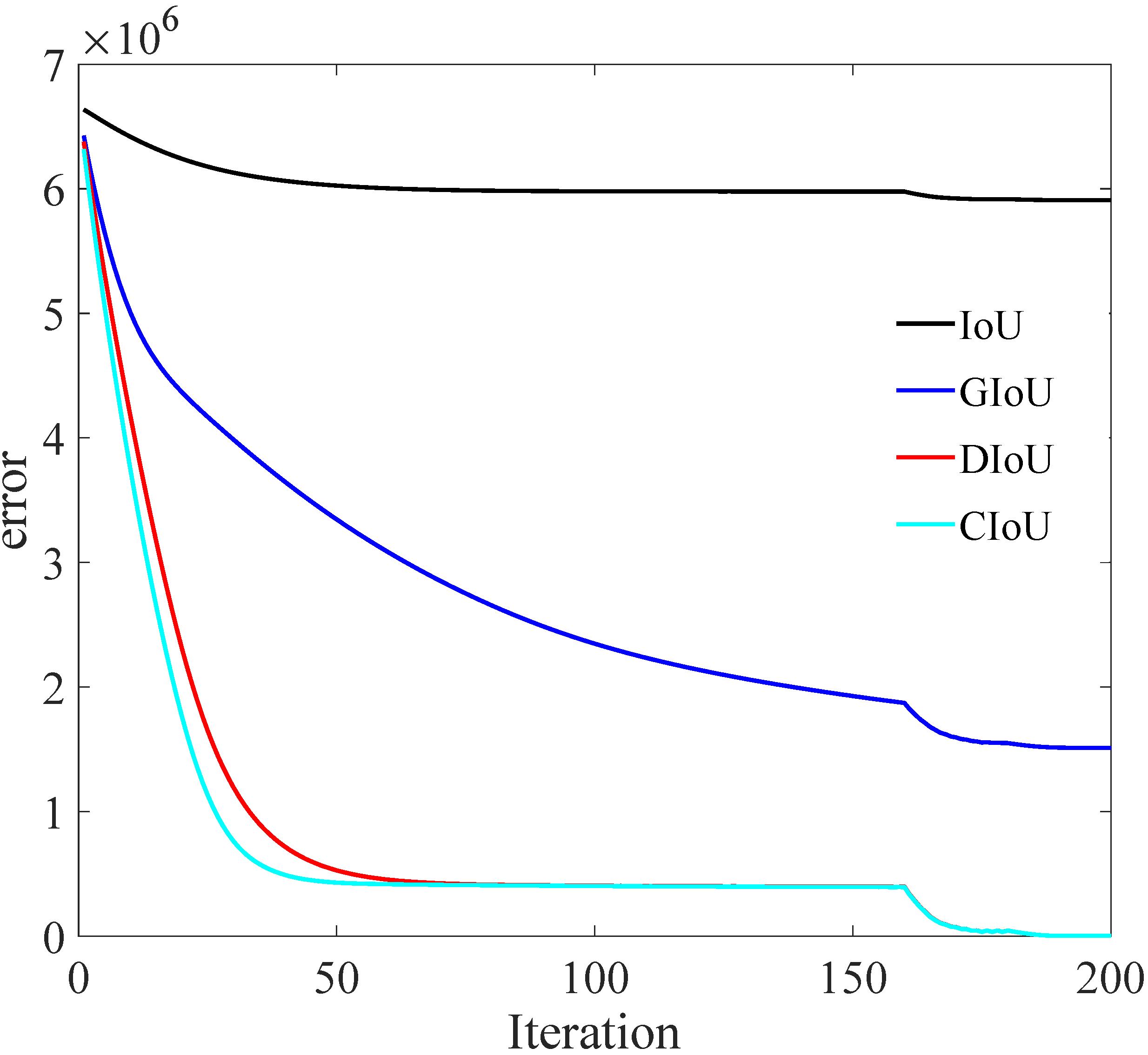}
	\caption{Simulation experiments: (a) 1,715,000 regression cases are adopted by considering different distances, scales and aspect ratios, (b) regression error sum (i.e., $\sum_{n} \bm{E}(t,n)$) curves of different loss functions at iteration $t$. }
	\label{fig:1715ksampling}
\end{figure}

\subsubsection{Limitations of IoU and GIoU Losses}
In Fig. \ref{fig:finalerror}, we visualize the final regression errors at iteration $T$ for 5,000 scattered points.
From Fig. \ref{fig:finalerror}(a), it is easy to see that IoU loss only works for the cases of overlapping with target boxes.
The anchor boxes without overlap will not move due to that the gradient is always 0.

By adding a penalty term as Eqn. \eqref{eq:giou loss}, GIoU loss can better relieve the issues of non-overlapping cases.
From Fig. \ref{fig:finalerror}(b), GIoU loss significantly enlarges the basin, i.e., the area that GIoU works.
But the cases  with extreme aspect ratios are likely to still have large errors.
This is because that the penalty term in GIoU loss is used to minimize $|C-A\cup B|$, but the area of $C-A\cup B$ is often small or 0 (when two boxes have inclusion relationships), and then GIoU almost degrades to IoU loss.
%
%
GIoU loss would converge to good solution as long as running sufficient iterations with proper learning rates, but the convergence rate is indeed very slow.
Geometrically speaking, from the regression steps as shown in Fig. \ref{fig:regression steps}, one can see that GIoU actually increases the predicted box size to overlap with target box, and then the IoU term will make the predicted box match with the target box, yielding a very slow convergence.
\begin{figure*}[!tb]
	\footnotesize
	\centering
	\begin{tabular}{cccccccccc}
		\includegraphics[width=0.52\columnwidth]{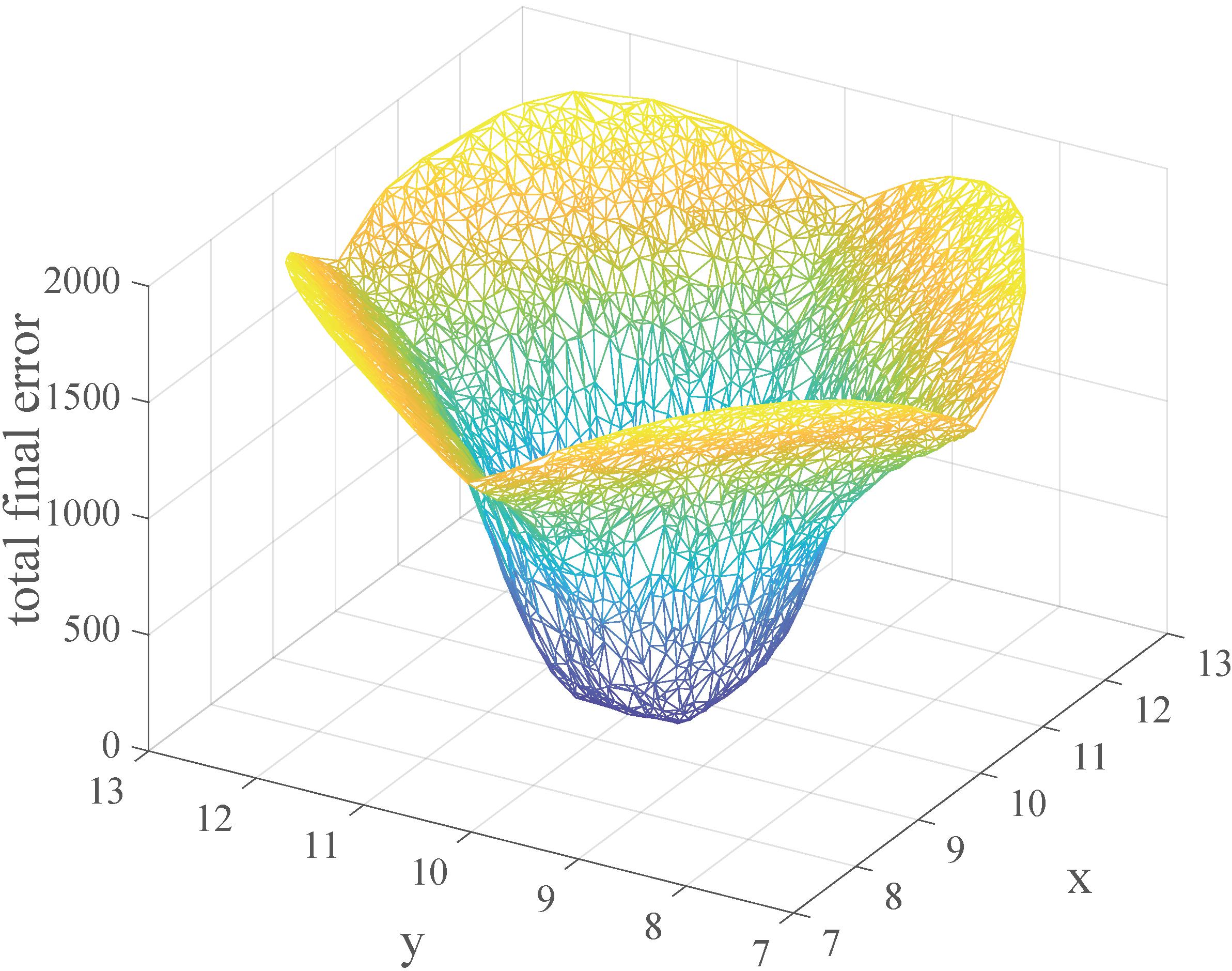}&
		\includegraphics[width=0.52\columnwidth]{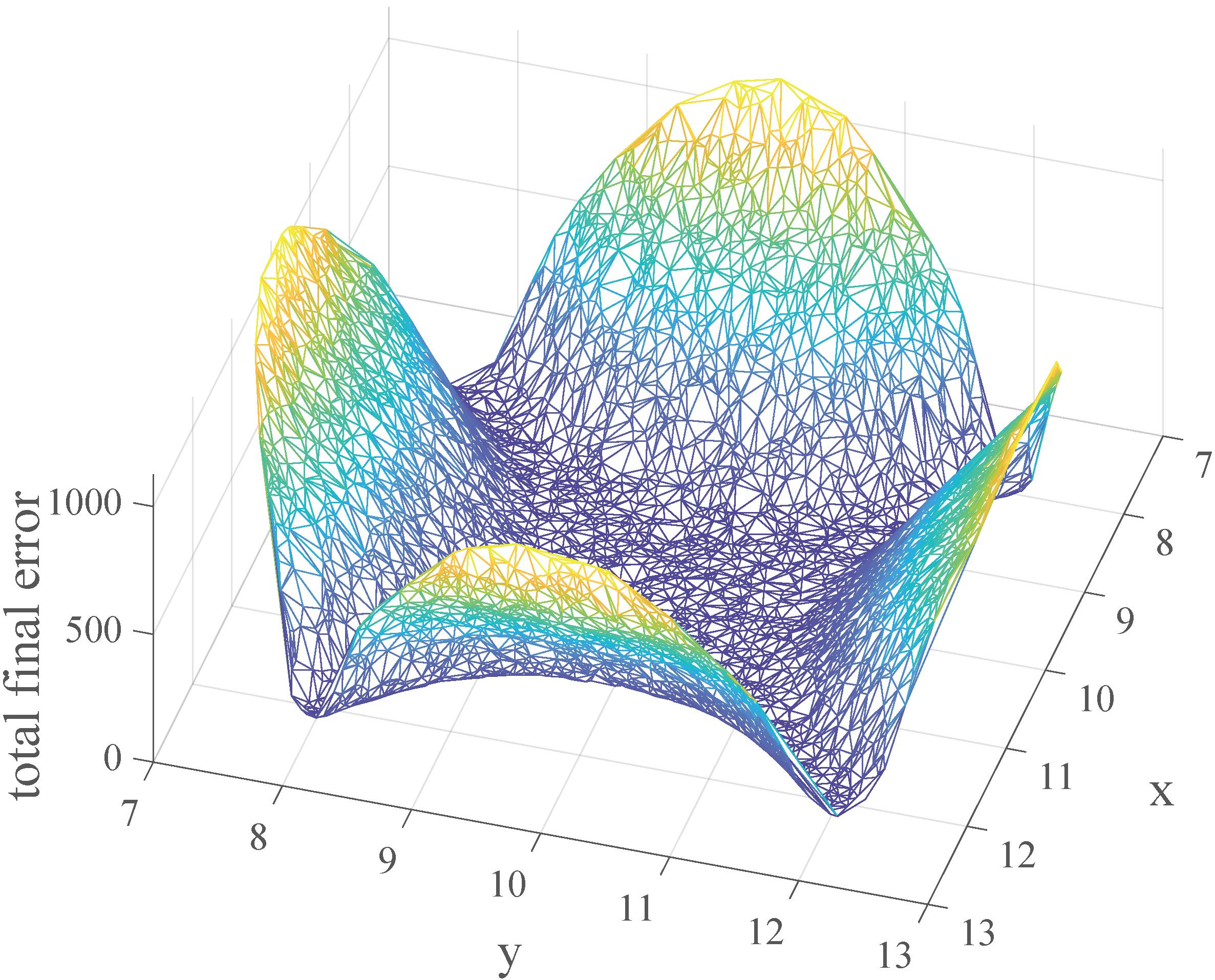}&
		\includegraphics[width=0.52\columnwidth]{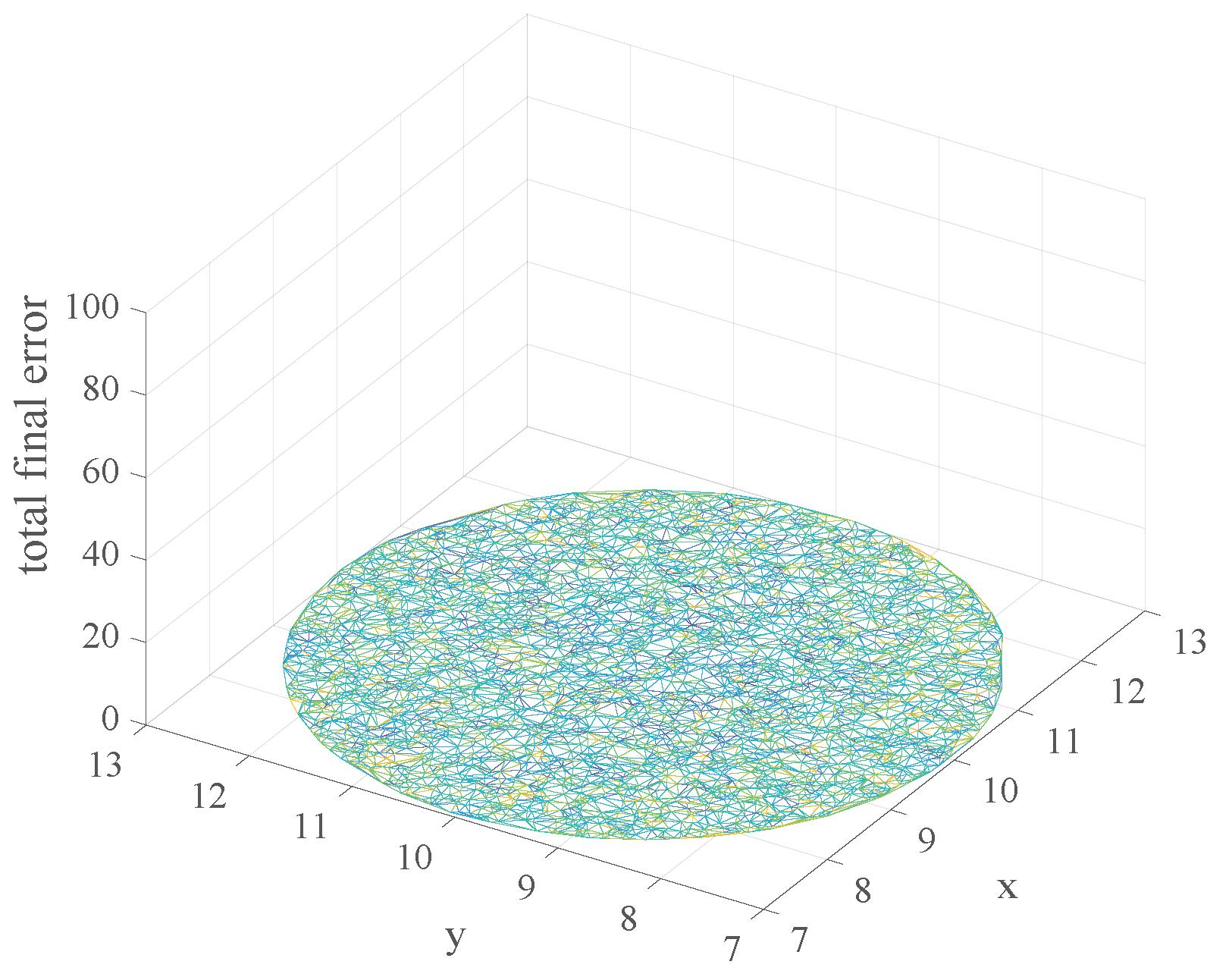}\\
		(a) $\mathcal{L}_{IoU}$ & (b) $\mathcal{L}_{GIoU}$ & (c) $\mathcal{L}_{CIoU}$
	\end{tabular}
	\caption{Visualization of regression errors of IoU, GIoU and CIoU losses at the final iteration $T$, i.e., $\bm{E}(T,n)$ for every coordinate $n$.
		We note that the basins in (a) and (b) correspond to good regression cases. One can see that IoU loss has large errors for non-overlapping cases, GIoU loss has large errors for horizontal and vertical cases, and our CIoU loss leads to very small regression errors everywhere.}
	\label{fig:finalerror}
\end{figure*}

To sum up, IoU-based losses only aim to maximize the overlap area of two boxes.
Original IoU loss converges to bad solutions for non-overlapping cases, while GIoU loss is with slow convergence especially for the boxes  with extreme aspect ratios.
And when incorporating into object detection or instance segmentation pipeline, both IoU and GIoU losses cannot guarantee the accuracy of regression.
We in this paper suggest that a good loss function for bounding box regression should enhance more geometric factors besides overlap area.

\subsection{CIoU Loss}


Considering the geometric factors for modeling regression relationships in Simulation Experiment, we suggest that a loss function should take three geometric factors into account, i.e., overlap area, distance and aspect ratio.
Generally, a complete loss can be defined as,
\begin{equation}\label{eq:reg}\small
\mathcal{L}= S(\mathcal{B},\mathcal{B}^{gt})+D(\mathcal{B},\mathcal{B}^{gt})+V(\mathcal{B},\mathcal{B}^{gt}),
\end{equation}
where $S$, $D$, $V$ denote the overlap area, distance and aspect ratio, respectively.
One can see that IoU and GIoU losses only consider the overlap area.
In the complete loss, IoU is only a good choice for $S$,
\begin{equation}\label{eq:ciou area}
\begin{array}{l}
S = 1-IoU.\\
\end{array}
\end{equation}
Similar to IoU, we want to make both $D$ and $V$ be also invariant to regression scale.
In particular, we adopt normalized central point distance to measure the distance of two boxes,
\begin{equation}\label{eq:ciou distance}
\begin{array}{l}
D = \frac{\rho^2(\bm{p},\bm{p}^{gt})}{c^2}, \\
\end{array}
\end{equation}
where $\bm{p} = [x,y]^\text{T}$ and $\bm{p}^{gt}=[x^{gt},y^{gt}]^\text{T}$ are the central points of boxes $\mathcal{B}$ and $\mathcal{B}^{gt}$, $c$ is the diagonal length of box $\mathcal{C}$, and $\rho$ is specified as Euclidean distance, as shown in Fig. \ref{fig:DIoU}.
\begin{figure}[!t]
	\centering
	\includegraphics[width=0.38\columnwidth]{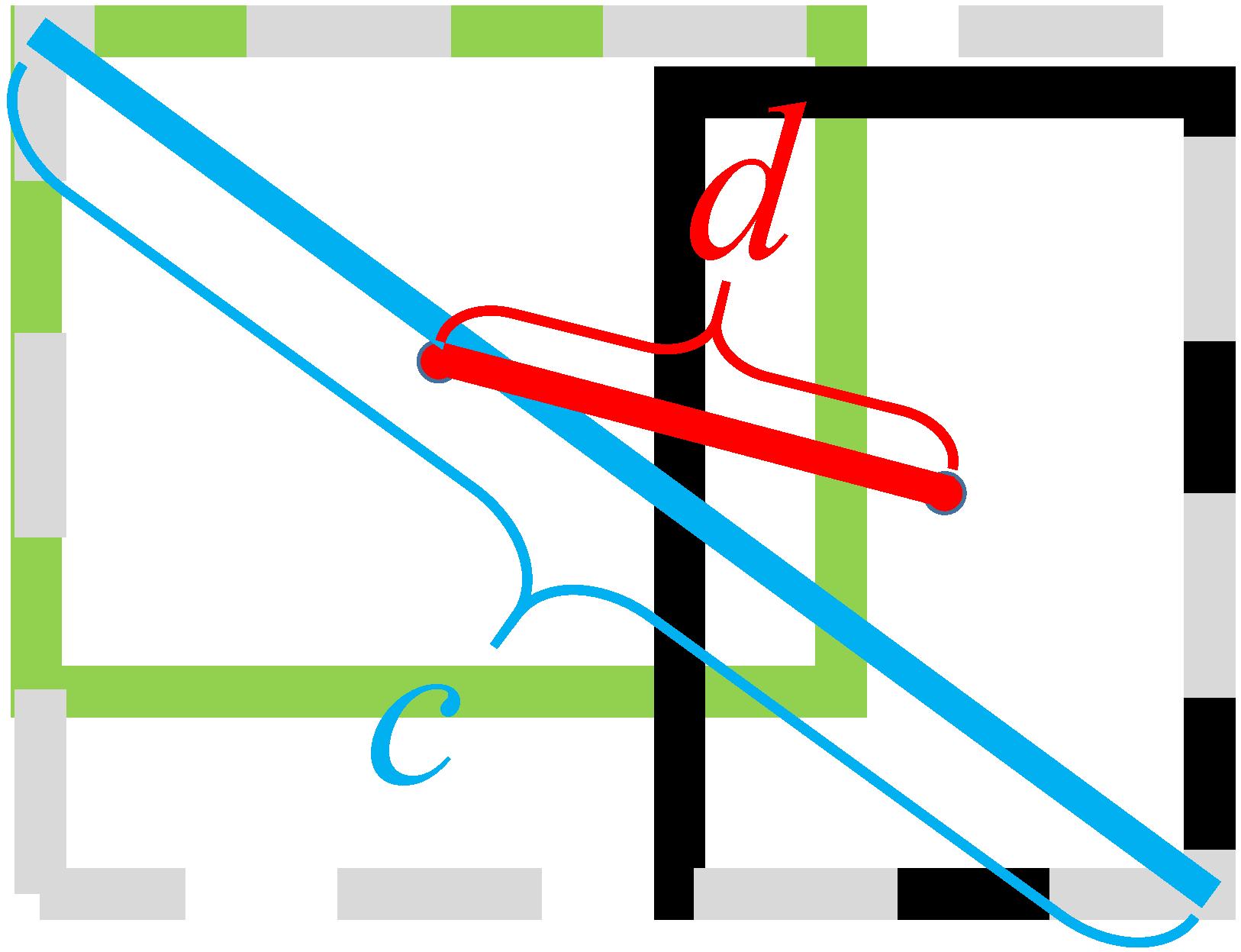}
	\caption{Normalized central point distance.
		$c$ is the diagonal length of the smallest enclosing box covering two boxes, and $d=\rho(\bm{p},\bm{p}^{gt})$ is the distance of central points of two boxes. }
	\label{fig:DIoU}
\end{figure}
And the consistency of aspect ratio is implemented as
\begin{equation}\label{three geometric}
\begin{array}{l}
V = \frac{4}{\pi^{2}}({\tt{arctan}}\frac{w^{gt}}{h^{gt}}-{\tt{arctan}}\frac{w}{h})^2.\\
\end{array}
\end{equation}
Finally, we obtain the Complete-IoU (CIoU) loss,
\begin{equation}\label{CIoU loss}
\mathcal{L}_{CIoU} = 1-IoU + \frac{\rho^2(\bm{p},\bm{p}^{gt})}{c^2} + \alpha V.
\end{equation}
It is easy to see that $S$, $D$ and $V$ are invariant to regression scale and are normalized to $[0,1]$.
Here, we only introduce one trade-off parameter $\alpha$, which is defined as
\begin{equation}
\alpha=\left\{
\begin{array}{ll}
0, &\text{if }IoU < 0.5,\\
\frac{V}{(1-IoU)+ V}, &\text{if }IoU\geq 0.5. \\
\end{array}\right.
\end{equation}
One can see that our CIoU loss will degrade to DIoU loss in our pioneer work \cite{diou} when $IoU<0.5$.
It is reasonable that when two boxes are not well matched, the consistency of aspect ratio is less important.
And when $IoU\geq 0.5$, the consistency of aspect ratio becomes necessary.

The proposed CIoU loss inherits some properties from IoU and GIoU losses.
(i)	CIoU loss is still invariant to the scale of regression problem.
(ii) Analogous to GIoU loss, CIoU loss can provide moving directions for bounding boxes when non-overlapping with target box.
Furthermore, our CIoU loss has two merits over IoU and GIoU losses, which can be evaluated by simulation experiment.
(i)	As shown in Fig. \ref{fig:regression steps} and Fig. \ref{fig:1715ksampling}, CIoU loss can rapidly minimize the distance of two boxes, and thus converges much faster than GIoU loss.
(ii) For the cases with inclusion of two boxes, or with extreme aspect ratios, CIoU loss can make regression very fast, while GIoU loss has almost degraded to IoU loss, i.e., $|\mathcal{C}-\mathcal{B}\cup \mathcal{B}^{gt}| \rightarrow 0$.


\section{Cluster-NMS}\label{sec:Cluster-NMS}

Most state-of-the-art object detection \cite{yolov3,SSD,fasterrcnn} and instance segmentation \cite{yolact} adopt the strategy to place more anchor boxes to detect difficult and small objects for improving detection accuracy.
Moreover, NMS for suppressing redundant boxes is facing tremendous pressure during inference.

Let $\bm{B} = [\mathcal{B}_1, \mathcal{B}_2,\cdots,\mathcal{B}_N]^{\text{T}}$
\footnote{Actually, $\bm{B}$ is a tensor with size $C\times N\times 4$, which contains $C$ classes. Since different classes share the same NMS operation, we omit this channel for simplicity. }
be an $N\times 4$ matrix storing $N$ detected boxes.
These boxes have been sorted according to the non-ascending classification scores, i.e., $s_1 \geq s_2\geq \cdots \geq s_N$.
Original NMS is to suppress redundant boxes by sequentially traversing $N$ boxes.
Specifically, for the box $\mathcal{M}$ with the current highest score, original NMS can be formally defined as,
\begin{equation}\label{eq:nms}
{{s}_{j}}=\left\{
\begin{aligned}
& {{s}_{j}},\text{ if } IoU(\mathcal{M}, \mathcal{B}_j)  < \varepsilon , \\
& 0,\ \text{  if } IoU(\mathcal{M}, \mathcal{B}_j) \ge \varepsilon , \\
\end{aligned} \right.
\end{equation}
where $\varepsilon$ is a threshold.
Original NMS is very time-consuming.
And several improved NMS, e.g., Soft-NMS \cite{softnms} and non-maximum weighted (Weighted-NMS) \cite{CAD}, can further improve the precision and recall, but they are much more inefficient.
We propose Cluster-NMS, where NMS can be parallelly done on implicit clusters of detected boxes, usually requiring less iterations.
Besides, Cluster-NMS can be purely implemented on GPU, and is much more efficient than original NMS.
Then, we incorporate the geometric factors into Cluster-NMS for further improving both precision and recall, while maintaining high inference efficiency.
%


\subsection{Cluster-NMS}

\begin{algorithm}[!tb]	
	
	\footnotesize
	\caption{Cluster-NMS}
	\label{alg:Cluster-NMS}
	\begin{algorithmic}[1]
		\small{	
			\Require{$N$ detected boxes $\bm{B} = [\mathcal{B}_1, \mathcal{B}_2,\cdots,\mathcal{B}_N]^{\text{T}}$ with non-ascending sorting by classification score, i.e., $s_1 \geq \cdots \geq s_N$.
			}
			\Ensure{$\bm{b}=\{b_i\}_{1\times N}, b_i\in \{0,1\}$ encodes final detection result, where $1$ denotes reservation and $0$ denotes suppression.}}
		\State Initialize $T=N$, $t=1$, $t^*=T$ and $\bm{b}^0=\mathbf{1}$
		\State Compute IoU matrix $\bm{X}=\{x_{ij}\}_{N\times N}$ with $x_{ij}=IoU(\mathcal{B}_i, \mathcal{B}_{j})$.
		\State $\bm{X}={\tt triu}(\bm{X})$ \Comment{Upper triangular matrix with $x_{ii} =0, \forall i$}
		\While {$t\leq T$}
		\State $\bm{A}^t = {\tt diag}(\bm{b}^{t-1})$
		\State $\bm{C}^t =\bm{A}^t\times \bm{X}$
		\State $\bm{g} \leftarrow \max\limits_j \bm{C}^t$ \Comment{Find maximum for each column $j$}
		\State $\bm{b}^t\leftarrow {\tt find}(\bm{g} <\varepsilon)$   \Comment{$\left\{\begin{array}{ll}
			b_{j} =1,& \text{ if } g_{j}<\varepsilon\\
			b_{j} =0,& \text{ if } g_{j}\geq \varepsilon\\
			\end{array}}\right.$
		
		\If {$\bm{b}^t==\bm{b}^{t-1}$}
		\State $t^*=t$, break
		\EndIf
		\State $t=t+1$
		\EndWhile
		\State\textbf{return} {$\bm{b}^{t^*}$}
	\end{algorithmic}
\end{algorithm}

We first compute the IoU matrix $\bm{X}=\{x_{ij}\}_{N\times N}$, where $x_{ij} =IoU(\mathcal{B}_i,\mathcal{B}_j)$.
We note that $\bm{X}$ is a symmetric matrix, and we only need the upper triangular matrix of $\bm X$, i.e., $\bm{X}={\tt triu}(\bm{X})$ with $x_{ii}=0, \forall i$.
In Fast NMS \cite{yolact}, the suppression is directly performed on the matrix $\bm X$, i.e., a box $\mathcal{B}_j$ would be suppressed as long as $\exists$ $x_{ij} > \varepsilon$.
Fast NMS is indeed efficient, but it suppresses too many boxes.
Considering $\exists \mathcal{B}_i$ has been suppressed, the box $\mathcal{B}_i$ should be excluded when making the rule for suppressing $\mathcal{B}_j$ even if $x_{ij}>\varepsilon$.
But in Fast NMS, $\mathcal{B}_i$ is actually taken into account, thereby being likely to over-suppress more boxes.

Let $\bm b = \{b_i\}_{1\times N}, b_i \in\{0,1\}$ be a binary vector to indicate the NMS result.
We introduce an iterative strategy, where current suppressed boxes with $b_i = 0$ would not affect the results.
Specifically, for iteration $t$, we have the NMS result $\bm{b}^{t-1}$, and introduce two matrices,
\begin{equation}
\begin{aligned}
\bm{A}^t &= {\tt diag}(\bm{b}^{t-1}),\\
\bm{C}^t &= \bm{A}^t \times \bm{X}.
\end{aligned}
\end{equation}
Then the suppression is performed on the matrix $\bm C$.
The details can be found in Alg. \ref{alg:Cluster-NMS}.
All these operations can be implemented on GPU, and thus Cluster-NMS is very efficient.

\begin{figure*}[!t]
	\centering
	\includegraphics[width=1\textwidth]{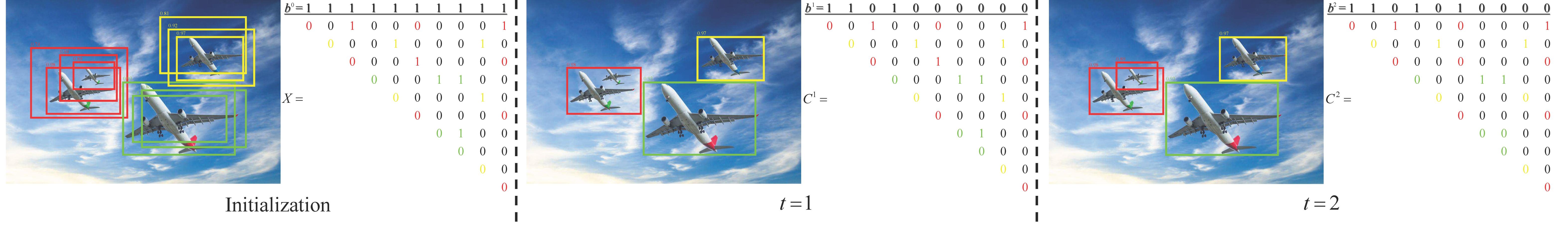}
	\caption{An example of Cluster-NMS, where 10 detected boxes are implicitly grouped into 3 clusters. The IoU matrix $\bm X$ has been binarized by threshold $\varepsilon=0.5$.
		When $t=1$, Cluster-NMS is equivalent to Fast NMS \cite{yolact}, where the boxes are over-suppressed.
		Theoretically, Cluster-NMS will stop after at most $4$ iterations, since the largest cluster (red boxes) has 4 boxes.
		But $\bm b$ after only $2$ iterations is exactly the final result of original NMS, indicating that Cluster-NMS usually requires less iterations.
	}
	\label{fig:earlystop}
\end{figure*}

When $T=1$, Cluster-NMS degrades to Fast NMS, and when $T=N$, the result of Cluster-NMS is totally same with original NMS.
Due to the pure operation on matrix $\bm A$ and $\bm X$, the predicted boxes without overlaps are implicitly grouped into different clusters, and the suppression is performed in parallel between clusters.
Cluster-NMS can guarantee the same result with original NMS, and generally can stop with less iterations, referring to Sec. \ref{sec:theoretical} for theoretical analysis.
%
%
We present an example in Fig. \ref{fig:earlystop}, where 10 detected boxes can be divided into 3 clusters and the largest cluster contains 4 boxes.
Thus, the maximum iteration $t^{*}$ of Cluster-NMS in Alg. \ref{alg:Cluster-NMS} is 4.
But in Fig. \ref{fig:earlystop}, $\bm b$ after only $2$ iterations is exactly the final result of original NMS.
In practical, Cluster-NMS usually requires less iterations.
%
%

We also note that original NMS has been implemented as CUDA NMS in Faster R-CNN \cite{fasterrcnn}, which is recently included into TorchVision 0.3.
TorchVision NMS is faster than Fast NMS and our Cluster-NMS (see Table \ref{tab:yolonms}), due to engineering accelerations.
Our Cluster-NMS requires less iterations and can also be further accelerated by adopting these engineering tricks, e.g., logic operations on binarized $\bm{X}$ as in Proposition \ref{prop1}.
But this is beyond the scope of this paper.
Moreover, the main contribution of Cluster-NMS is that geometric factors can be easily incorporated into Cluster-NMS, while maintaining high efficiency.

\subsection{Incorporating Geometric Factors into Cluster-NMS}

Geometric factors for measuring bounding box regression can be introduced into Cluster-NMS for improving precision and recall.

\subsubsection{{Score Penalty Mechanism into Cluster-NMS}}
Instead of the hard threshold in original NMS Eqn. \eqref{eq:nms}, we introduce a Score Penalty Mechanism based on the overlap areas into Cluster-NMS, analogous to Soft-NMS.
Specifically, in Cluster-NMS$_{S}$, we adopt the score penalty following "Gaussian" Soft-NMS \cite{softnms}, and the score $s_j$ is re-weighted as
\begin{equation}
s_j=s_j\prod\limits_i e^{-\frac{(\bm{A}\times\bm{X})_{ij}^2}{\sigma}},
\end{equation}
where $\sigma = 0.2$ in this work. 
It is worth noting that our Cluster-NMS$_{S}$ is not completely same with Soft-NMS.
In Cluster-NMS$_S$, $s_j$ is only penalized by those boxes with higher scores than $s_j$, since $\bm{A} \times \bm{X}$ is an upper triangular matrix.
%

\subsubsection{{Normalized Central Point Distance into Cluster-NMS}}
As suggested in our pioneer work \cite{diou}, normalized central point distance can be included into NMS to benefit the cases with occlusions.
By simply introducing the normalized central point distance $D$  in Eqn. \eqref{eq:ciou distance} into IoU matrix $\bm X$, Cluster-NMS$_D$ is actually equivalent with DIoU-NMS in our pioneer work \cite{diou}.
Moreover, we can incorporate the normalized central point distance into Cluster-NMS$_{S}$, forming Cluster-NMS$_{S+D}$.
Specifically, the score $s_j$ is penalized as
\begin{equation}
s_j=s_j\prod\limits_i \min\{e^{-\frac{(\bm{A}\times \bm{X})_{ij}^2}{\sigma}}+D^{\beta},1\},
\end{equation}
where $\beta = 0.6$ is a trade-off parameter for balancing the precision and recall (see Fig. \ref{fig:DmNMS}).

\subsubsection{{Weighted Coordinates into Cluster-NMS}}
Weighted-NMS \cite{CAD} is a variant of original NMS.
Instead of deleting redundant boxes, Weighted-NMS creates new box by merging box coordinates according to the weighted combination of the boxes based on their scores and overlap areas.
The formulation of Weighted-NMS is as follow,
\begin{equation}
\mathcal{B}=\frac{1}{\sum\limits_jw_j}\sum\limits_{\mathcal{B}_j\in\Lambda}w_j \mathcal{B}_j,
\end{equation}
where $\Lambda=\{\mathcal{B}_j \ |\ x_{ij}\geq \varepsilon, \forall i\}$ is a set of boxes, weight $w_j=s_jIoU(\mathcal{B},\mathcal{B}_j)$, and $\mathcal{B}$ is the newly created box.
However, Weighted-NMS is very inefficient because of the sequential operations on every box.

Analogously, such weighted strategy can be included into our Cluster-NMS.
In particular, given the matrix $\bm C$, we multiply its every column using the score vector $\bm{s} = [s_1,s_2,\cdots,s_N]^{\text{T}}$ in the entry-by-entry manner, resulting in $\bm{C}'$ to contain both the classification score and IoU.
Then for the $N$ detected boxes $\bm{B}=[\mathcal{B}_1,\mathcal{B}_2,\cdots,\mathcal{B}_N]^{\text{T}}$, their coordinates can be updated by
\begin{equation}
\bm{B}=\frac{\bm{C}'\times \bm{B}}{\textit{Repmat}_4(\sum_i\bm{C}'(i,:))},
\end{equation}
where $\textit{Repmat}_4$ copies the $N\times 1$ vector 4 times to form $N\times4$ matrix, making the entry-wise division feasible.
The Cluster-NMS$_W$ shares the same output form with Cluster-NMS, but the coordinates of boxes have been updated.
Moreover, the normalized central point distance can be easily assembled into Cluster-NMS$_{W}$, resulting in Cluster-NMS$_{W+D}$, where the IoU matrix $\bm X$ is computed by considering both overlap area and distance as DIoU-NMS \cite{diou}.
Cluster-NMS$_W$ has the same result with Weighted-NMS as well as 6.1 times faster efficiency, and Cluster-NMS$_{W+D}$ contributes to consistent improvements in both AP and AR (see Table \ref{tab:yolonms}).

\subsection{Theoretical Analysis} \label{sec:theoretical}
In the following, we first prove that Cluster-NMS with $T=N$ iterations can achieve the same suppression result with original NMS, and then discuss that Cluster-NMS usually requires less iterations.

\begin{proposition}\label{prop1}
	Let $\bm{b}^T$ be the result of Cluster-NMS after $T$ iterations, $\bm{b}^T$ is also the final result of original NMS.
\end{proposition}
\begin{proof}
	
	Let $\bm{X}_k$ be the square block matrix of $\bm{X}$ containing the first $k\times k$ partition, while let $\bm{X}_{N-k}$ be the square block matrix of $\bm{X}$ containing the last $(N-k)\times(N-k)$ partition.
	The matrices $\bm A$ and $\bm C$ share the same definition.
	Let $\bm{b}_k^k$ be the subvector of $\bm{b}$ containing the first $k$ elements after $k$ iterations in Cluster-NMS.
	And it is straightforward that $\bm{A}_k = {\tt diag}(\bm{b}_k^k)$.
	Besides, we binarize the upper triangular IoU matrix $\bm{X}=\{x_{ij}\}_{N\times N}$ by threshold $\varepsilon$,
	\begin{equation}
		\left\{\begin{array}{ll}
			x_{ij} =0,& \text{ if } x_{ij}<\varepsilon,\\
			x_{ij} =1,& \text{ if } x_{ij}\geq \varepsilon,\\
		\end{array}\right.
	\end{equation}
	which does not affect the result by Cluster-NMS in Alg. \ref{alg:Cluster-NMS}, but makes the proof easier to understand.
	
	
	(\emph{\textbf{i}}) When $t=1$, $\bm{b}^1_1=[b_1]^{\text{T}}=[1]^{\text{T}}$ is same to the result by original NMS, since the first box with the largest score is kept definitely.
	When $t=2$, we have
	\begin{equation}
		\bm{C}=\bm{A}\times \bm{X} = \left(\begin{array}{cc}
			\bm{A}_{2} &\mathbf{0} \\ \mathbf{0} & \bm{A}_{N-2}\\ \end{array}\right)
		\left(\begin{array}{cc}
			\bm{X}_{2} & \bm{X}_{other} \\ \mathbf{0} & \bm{X}_{N-2}\\ \end{array}\right).\end{equation}
	Since the result $\bm{b}_2^2$ is not affected by the last $N-k$ boxes, we only consider
	\begin{equation}
		\bm{C}_2=\bm{A}_2\times \bm{X}_2 = \left(
		\begin{array}{cc} 0 & b_1x_{1,2} \\ 0 & 0\\ \end{array}\right)=\left(\begin{array}{cc} 0 & x_{1,2} \\ 0 & 0\\ \end{array}\right),
	\end{equation}
	where $x_{1,2}$ is a binary value.
	And thus $\bm{b}^2_2=[b_1,b_2]^{\text{T}}$, where $b_2=\lnot x_{1,2}$ is exactly the output of original NMS at iteration $t=2$.
	
	(\emph{\textbf{ii}}) Then we assume that when $t=k$, $\bm{b}^k_k$ determined by $\bm{C}_k$ is same with the result of original NMS after iteration $t=k$.
	When $t=k+1$, we have
	\begin{equation}
		\bm{C}=\bm{A}\times \bm{X} = \left(\begin{array}{cc}
			\bm{A}_{k+1} &\mathbf{0} \\ \mathbf{0} & \bm{A}_{N-k-1}\\ \end{array}\right)
		\left(\begin{array}{cc}
			\bm{X}_{k+1} & \bm{X}_{other} \\ \mathbf{0} & \bm{X}_{N-k-1}\\ \end{array}\right).\end{equation}
	Analogously, we do not care these block matrices $\bm{A}_{N-k-1}, \bm{X}_{other} \text{ and } \bm{X}_{N-k-1}$, since they do not affect $\bm{C}_{k+1}$,
	\begin{equation}\begin{aligned}
			\bm{C}_{k+1}= \bm{A}_{k+1}\times \bm{X}_{k+1}&=\left(\begin{array}{cc} \bm{A}_k & 0 \\ \mathbf{0} & b\\
			\end{array}\right)
			\left(\begin{array}{cc}
				\bm{X}_k & \bm{X}_{:,k+1} \\ \mathbf{0} & 0\\
			\end{array}\right),\\
			&=\left(\begin{array}{cc}
				\bm{C}_k & \bm{A}_k\times \bm{X}_{:,k+1} \\ \mathbf{0} & 0\\
			\end{array}\right),
	\end{aligned}\end{equation}
	where $\bm{X}_{:,k+1}$ is the $k$-th column of $\bm{X}_k$ by excluding the last 0.
	Then it is easy to see that $b_{k+1}$ can be determined by
	\begin{equation}
		b_{k+1}=\lnot \max (\bm{A}_k \times \bm{X}_{:,k+1}),
	\end{equation}
	which is the output of original NMS at iteration $t=k+1$.
	And thus $\bm{b}^{k+1}_{k+1} = [\bm{b}_k^k; b_{k+1}]$ is same with the result of original NMS after iteration $t=k+1$.
	
	Combining (\emph{\textbf{i}}) and (\emph{\textbf{ii}}), it can be deduced that $\bm{b}^T$ after $T$ iterations in Cluster-NMS is exactly the final result of original NMS.
\end{proof}

\emph{\textbf{Discussion}}:
Cluster-NMS usually requires less iterations.
%
Let $\mathbb{B}^* = \{\mathcal{B}_{j_1},\mathcal{B}_{j_2},\cdots,\mathcal{B}_{j_M}\}$ be the largest cluster containing $M$ boxes, where a box $\mathcal{B}_j \in \mathbb{B}^*$ if and only if $\exists i \in \{j_1,j_2,\cdots,j_M\}, IoU(\mathcal{B}_i,\mathcal{B}_j) \geq \varepsilon$, and $IoU(\mathcal{B}_j,\mathcal{B}_u) < \varepsilon, \forall j \in \{j_1,j_2,\cdots,j_M\}$ and $\forall u \notin \{j_1,j_2,\cdots,j_M\}$.
Thus for the binarized IoU matrix $\bm X$, $x_{u,j}=0$, $\forall j \in \{j_1,j_2,\cdots,j_M\}$ and $\forall u\notin \{j_1,j_2,\cdots,j_M\}$.
That is to say $\mathbb{B}^*$ is actually processed without considering boxes in other clusters, and after $M$ iterations, $\bm b$ definitely will not change.
Different clusters are processed in parallel by Cluster-NMS.

\section{Experimental Results}\label{sec:experiment}
In this section, we evaluate our proposed CIoU loss and Cluster-NMS for state-of-the-art instance segmentation YOLACT \cite{yolact} and BlendMask-RT \cite{blendmask}, and object detection YOLO v3 \cite{yolov3}, SSD \cite{SSD} and Faster R-CNN \cite{fasterrcnn} on popular benchmark datasets, e.g., PASCAL VOC \cite{voc} and MS COCO \cite{coco}.
%
CIoU loss is implemented in C/C++ and Pytorch, and Cluster-NMS is implemented in Pytorch.
For the threshold $\varepsilon$, we adopt their default settings.
The source code and trained models have been made publicly available.

\begin{table*}\footnotesize
	\centering
	\setlength{\tabcolsep}{2pt}
	\caption{Instance segmentation results of YOLACT \cite{yolact}. The models are re-trained using Smooth-$\ell_1$ loss and CIoU loss by us, and the results are reported on MS COCO 2017 validation set \cite{coco}. }
	\newcommand{\modelname}[1]{\methodname{}-#1}
	\begin{spacing}{1.2}
		\begin{tabular}{c c| c |cc| cc cc| ccc c ccc c| ccc c ccc}
			\hline
			Method && Loss              && NMS Strategy &&    FPS     &     Time   &&  AP  & AP$_{50}$ & AP$_{75}$ &&  AP$_{S}$ &  AP$_{M}$ &  AP$_{L}$ && AR$_1$ & AR$_{10}$ & AR$_{100}$ && AR$_{S}$ & AR$_{M}$ &  AR$_{L}$ \\
			\hline
			&& Smooth-$\ell_1$   &&  Fast NMS    && {\bf 30.6}&{\bf 32.7}   &&   29.1   &   47.4   &   30.5   &&   9.4   &    32.0  &    48.5  &&   27.5   &    39.2  &   40.3   &&   18.8   &   44.7   &    59.8   \\
			\cline{3-24}
			&&\multirow{5}{*}{$\calL_{CIoU}$}
			&&    Fast NMS      &&{\bf 30.6}&{\bf 32.7}&&   29.6   &   48.1   &   30.9   &&   9.4   &    32.0  &    49.7  &&   27.6   &    39.3  &   40.3   &&   18.0   &   44.0   &    60.8   \\
			\multirow{2}{*}{YOLACT-550 \cite{yolact}}&&&&  Original NMS  &&   11.5   &   86.6   &&   29.7   &   48.3   &   31.0   &&   9.4   &    32.2  &    49.8  &&   27.5   &    40.1  &   41.7   &&   18.7   &   45.8   &    62.8   \\
			\cline{4-24}
			\multirow{2}{*}{R-101-FPN}
			&&&&  Cluster-NMS    &&   28.8   &   34.7   &&   29.7   &   48.3   &   31.0   &&   9.4   &    32.2  &    49.7  &&   27.5   &    40.1  &   41.7   &&   18.8   &   45.8   &    62.8   \\
			&&&& Cluster-NMS$_{S}$ &&   28.6   &   35.0   &&{\bf 30.3}&{\bf 49.1}&{\bf 31.7}&&{\bf 9.7} &{\bf 33.0}&{\bf 50.8}&&{\bf 27.7}&{\bf 41.4}&   43.6   &&{\bf 19.7}&   47.7   &    65.9   \\
			&&&& Cluster-NMS$_{S+D}$&&   27.1   &   36.9   && 30.2&   48.9   &{\bf 31.7}&& 9.6&            32.8&   50.7  &&27.6&    41.3   &{\bf 43.8}&&  19.5   &{\bf 47.8}&{\bf 66.4} \\
			\hline
		\end{tabular}
		
		\label{tab:clusternms1}
	\end{spacing}
\end{table*}

\begin{table*}\footnotesize
	\centering
	\setlength{\tabcolsep}{2pt}
	\caption{Instance segmentation results of YOLACT \cite{yolact}. The model is borrowed from the original paper \cite{yolact}, and the results are reported on MS COCO 2017 validation set \cite{coco}.}
	\newcommand{\modelname}[1]{\methodname{}-#1}
	\begin{spacing}{1.2}
		\begin{tabular}{l c| l |cc| cc cc| ccc c ccc c |ccc c ccc}
			\hline
			Method && Backbone               && NMS Strategy &&    FPS     &     Time   &&  AP  & AP$_{50}$ & AP$_{75}$ &&  AP$_{S}$ &  AP$_{M}$ &  AP$_{L}$ && AR$_1$ & AR$_{10}$ & AR$_{100}$ && AR$_{S}$ & AR$_{M}$ &  AR$_{L}$ \\
			\hline
			\multirow{5}{*}{YOLACT-550 \cite{yolact}} && \multirow{5}{*}{R-101-FPN}
			&&    Fast NMS      &&{\bf 30.6}&{\bf 32.7}&&   29.8   &   48.3   &   31.3   &&   10.1   &    32.2  &    50.1  &&   27.8   &    39.6  &   40.8   &&   18.9   &   44.8   &    61.0   \\
			&&&&  Original NMS  &&   11.9   &   83.8   &&   29.9   &   48.4   &   31.4   &&   10.0   &    32.3  &    50.3  &&   27.7   &    40.4  &   42.1   &&   19.5   &   46.4   &    62.8   \\
			\cline{4-24}
			&&&&  Cluster-NMS    &&   29.2   &   34.2   &&   29.9   &   48.4   &   31.4   &&   10.0   &    32.3  &    50.3  &&   27.7   &    40.4  &   42.1   &&   19.5   &   46.4   &    62.8   \\
			&&&& Cluster-NMS$_{S}$ &&   28.8   &   34.7   &&{\bf 30.5}&{\bf 49.3}&{\bf 32.1}&&{\bf 10.3}&{\bf 33.1}&{\bf 51.2}&&{\bf 27.8}&{\bf 41.8}&   44.1   &&   20.4   &   48.3   &    66.3   \\
			&&&& Cluster-NMS$_{S+D}$&&   27.5   &   36.4   &&   30.4   &   49.1   &   32.0   &&   10.2   &    32.9  &    51.1  &&{\bf 27.8}&    41.7  &{\bf 44.3}&&{\bf 20.5}&{\bf 48.5}&{\bf 66.8} \\
			\hline
		\end{tabular}
		
		\label{tab:clusternms2}
	\end{spacing}
\end{table*}

\begin{figure*}[!htb]\footnotesize
	\centering
	\includegraphics[width=0.85\textwidth]{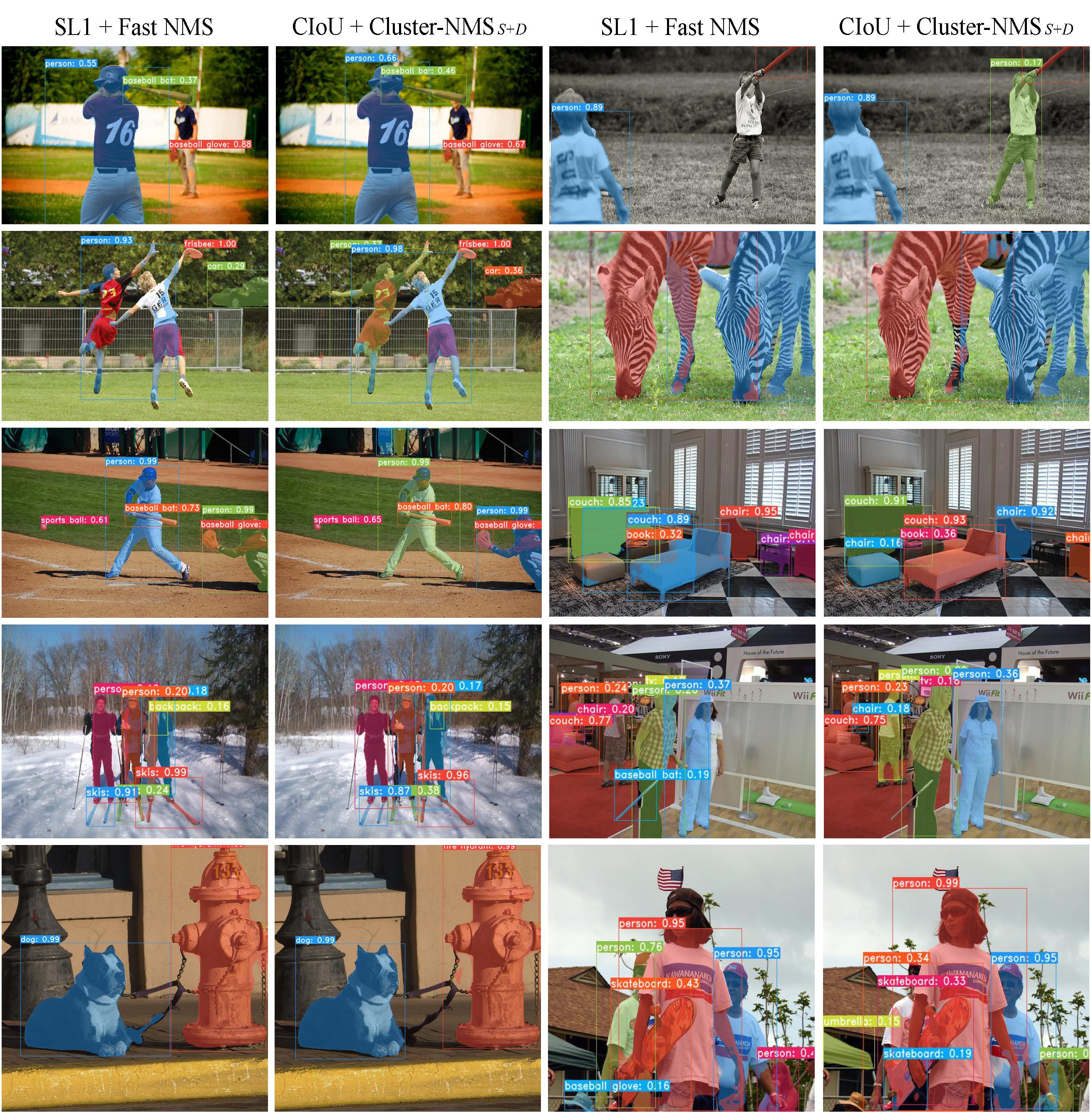}
	\caption{Detection and segmentation results of YOLACT \cite{yolact} on MS COCO 2017.}
	\label{fig:yolact}
\end{figure*}
\subsection{Instance Segmentation}
\subsubsection{YOLACT}

YOLACT \cite{yolact} is a real-time instance segmentation method, in which Smooth-$\ell_1$ loss is adopted for training, and Fast NMS is used for real-time inference.
To make a fair comparison, we train two YOLACT\footnote{\url{https://github.com/dbolya/yolact}} (ResNet-101-FPN) models using Smooth-$\ell_1$ loss and CIoU loss, respectively.
The training is carried out on NVIDIA GTX 1080Ti GPU with batchsize 4 per GPU.
The training dataset is MS COCO 2017 train set, and the testing dataset is MS COCO 2017 validation set.
The evaluation metrics include AP, AR, inference time (ms) and FPS.
The details of different settings of AP and AR can be found in \cite{SSD}.

The comparison results are reported in Table \ref{tab:clusternms1}. 
By adopting the same NMS strategy, i.e., Fast NMS, one can see that CIoU loss is superior to Smooth-$\ell_1$ loss in terms of most AP and AR metrics. 
Then on the YOLACT model trained using our CIoU loss, we evaluate the effectiveness of Cluster-NMS by assembling different geometric factors.
One can see that:
(i) In comparison to original NMS, Fast NMS is efficient but yields notable drops in AP and AR, while our Cluster-NMS can guarantee exactly the same results with original NMS and its efficiency is comparable with Fast NMS.
(ii) By enhancing score penalty based on overlap areas, Cluster-NMS$_S$ achieves notable gains in both AP and AR.
Especially, for large objects, Cluster-NMS$_S$ performs much better in AP$_L$ and AR$_L$.
The hard threshold strategy in original NMS is very likely to treat large objects with occlusion as redundancy, while our Cluster-NMS$_S$ is friendly to large objects.
(iii) By further incorporating distance, Cluster-NMS$_{S+D}$ achieves higher AR metrics, albeit its AP metrics are only on par with Cluster-NMS$_S$.
From Fig. \ref{fig:DmNMS}, one can see that distance is a crucial factor for balancing precision and recall, and we choose $\beta=0.6$ in our experiments.
(iv) Incorporating geometric factors into Cluster-NMS takes only a little more inference time, and $\sim$28 FPS on one GTX 1080Ti GPU can guarantee real-time inference.
To sum up, our Cluster-NMS with geometric factors contributes to significant performance gains, while maintaining high inference efficiency.

One may notice that our re-trained YOLACT model using Smooth-$\ell_1$ loss is a little inferior to the results reported in their original paper \cite{yolact}.
This is because batchsize in \cite{yolact} is set as 8, which causes out of memory on our GTX 1080Ti GPU.
%
Then, we also evaluate Cluster-NMS on their released models trained using Smooth-$\ell_1$ loss.
From {Table \ref{tab:clusternms2}}, one can draw the consistent conclusion that Cluster-NMS with geometric factors contributes to AP and AR improvements.
Considering Tables \ref{tab:clusternms1} and \ref{tab:clusternms2}, Cluster-NMS$_{S+D}$ is a better choice to balance precision and recall.
From Fig. \ref{fig:DmNMS}, it is easy to see that Cluster-NMS$_{S+D}$ can achieve higher precision by setting larger $\beta$, or can achieve higher recall by setting smaller $\beta$. 
Finally, we present the results of detection and segmentation in Fig. \ref{fig:yolact}, from which one can easily find the more accurate detected boxes and segmentation masks by our CIoU loss and Cluster-NMS$_{S+D}$.

\begin{figure}[!t]\footnotesize
	\centering
	\begin{tabular}{cccccccccc}
		\includegraphics[width=0.36\textwidth]{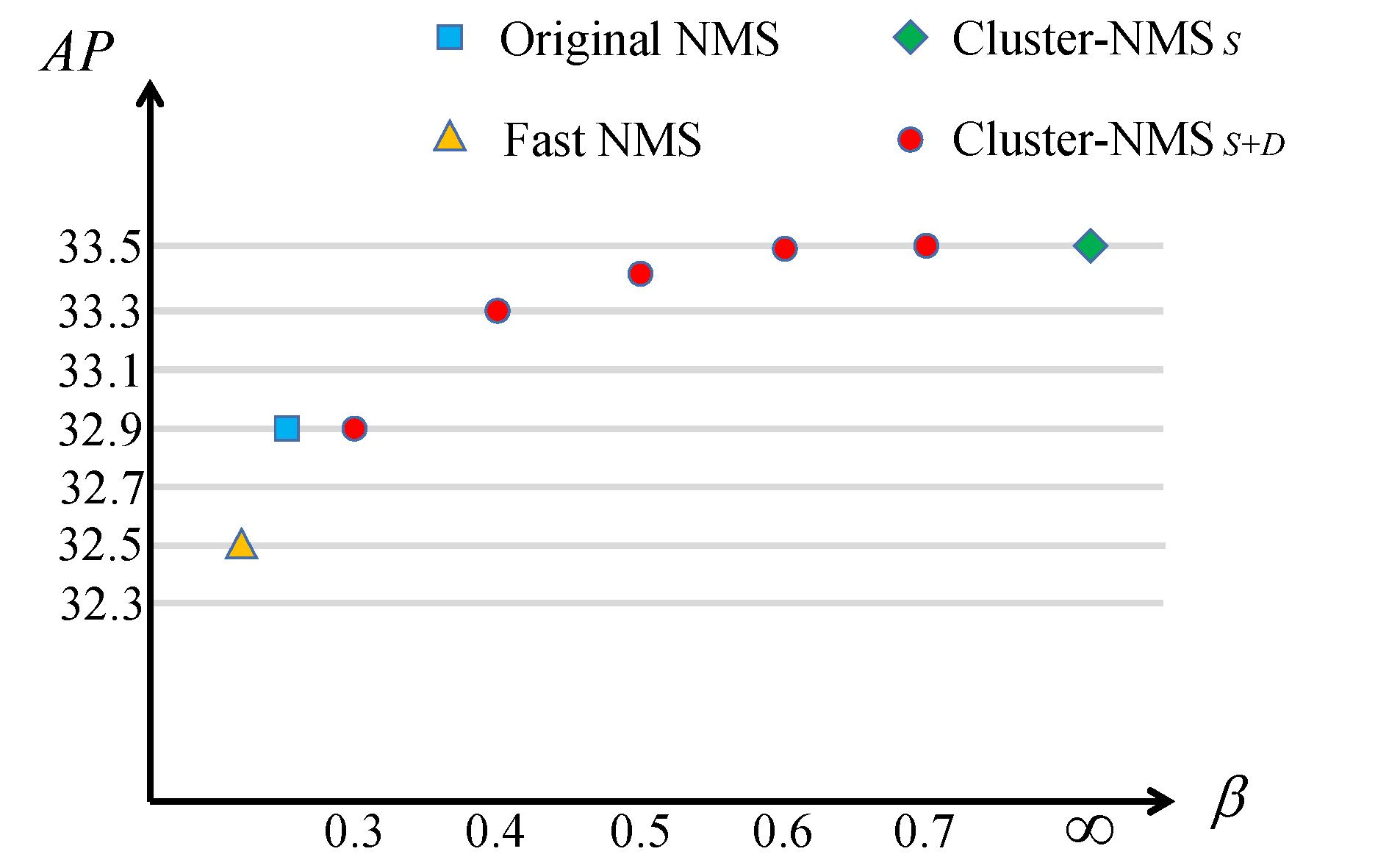}\\
		(a) Average Precision with different $\beta$ values\\
		\\
		\includegraphics[width=0.36\textwidth]{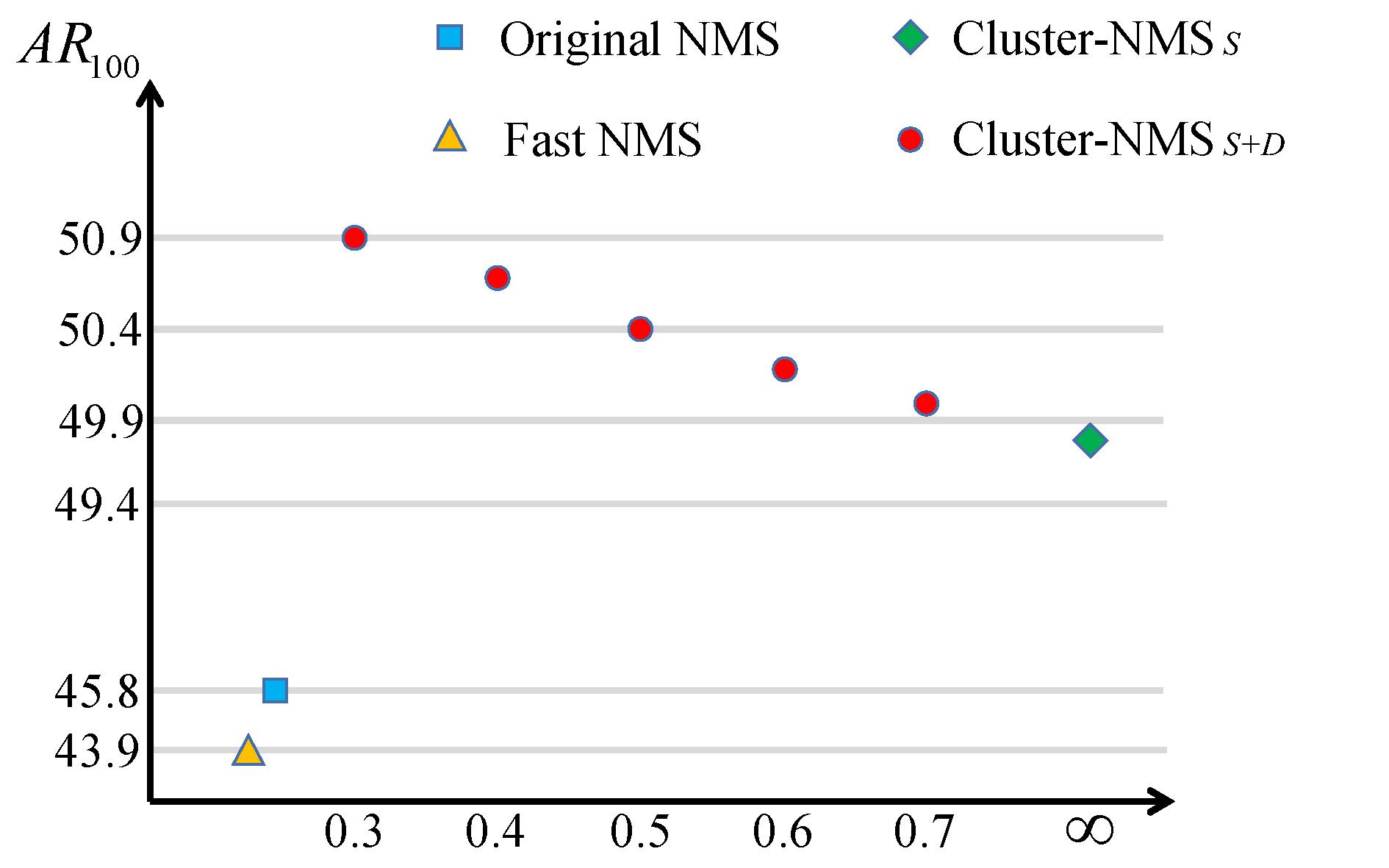}\\
		(b) Average Recall with different $\beta$ values\\
	\end{tabular}
	\caption{Balancing precision and recall by different values of $\beta$ in Cluster-NMS$_{S+D}$. The results are from YOLACT model on MS COCO 2017 validation set.}
	\label{fig:DmNMS}
\end{figure}

\subsubsection{BlendMask-RT}
We then evaluate our CIoU loss and Cluster-NMS on the most state-of-the-art instance segmentation method BlendMask-RT \cite{blendmask}, which is based on YOLACT by introducing attention mechanism.
To make a fair comparison, we train two BlendMask-RT\footnote{\url{https://github.com/aim-uofa/AdelaiDet}} (ResNet-50-FPN) models using GIoU loss and CIoU loss, respectively.
In experiments, we adopt the same training and testing sets with YOLACT and also the same evaluation measures.
Table \ref{tab:blendsegmentation} reports the quantitative results of BlendMask-RT, in which we can draw the same conclusion that CIoU loss is effective for training, and Cluster-NMS with geometric factors can further improve the performance, while maintaining high inference efficiency.
	\begin{table*}\footnotesize
		\centering
		\setlength{\tabcolsep}{1.5pt}
		\caption{Instance segmentation results of BlendMask-RT \cite{blendmask}. The models are re-trained using GIoU loss and CIoU loss by us, and the results are reported on MS COCO 2017 validation set. }
		\newcommand{\modelname}[1]{\methodname{}-#1}
		\begin{spacing}{1.2}
			\begin{tabular}{c c| c |cc| cc cc| ccc c ccc c |ccc c ccc}
				\hline
				Method && Loss               && NMS Strategy &&    FPS     &     Time   &&  AP  & AP$_{50}$ & AP$_{75}$ &&  AP$_{S}$ &  AP$_{M}$ &  AP$_{L}$ && AR$_1$ & AR$_{10}$ & AR$_{100}$ && AR$_{S}$ & AR$_{M}$ &  AR$_{L}$ \\
				\hline
				&& $\mathcal{L}_{GIoU}$ &&  Fast NMS    && {\bf 42.7}&{\bf 23.4}   && 34.4 & 54.3 & 36.3 && {\bf14.4} & 38.3 & 51.9 && 29.3 & 45.2 & 46.8 && 25.6 & 53.0 & 63.4 \\
				\cline{3-24}
				&&\multirow{5}{*}{$\mathcal{L}_{CIoU}$}
				&&    Fast NMS      &&{\bf 42.7}&{\bf 23.4}&& 34.9 & 55.2 & 37.0 && 14.1 & 38.9 & 51.9 && 29.6 & 45.3 & 47.0 && 25.5 & 53.2 & 63.8 \\
				\multirow{1}{*}{BlendMask-RT}&&&&  Original NMS (TorchVision)  && {\bf 42.7}&{\bf 23.4}  && 35.1 & {\bf55.4} & 37.2 && 14.2 & 39.1 & 52.2 && 29.6 & 46.3 & 48.5 && 26.4 & 55.0 & 65.9 \\
				\cline{4-24}
				\multirow{2}{*}{R-50-FPN}
				&&&&  Cluster-NMS    &&   40.7   &   24.6   && 35.1 & {\bf55.4} & 37.2 && 14.2 & 39.1 & 52.2 && 29.6 & 46.3 & 48.5 && 26.4 & 55.0 & 65.9 \\
				&&&& Cluster-NMS$_W$ &&   39.7   &   25.2   && {\bf35.2} & {\bf55.4} & 37.3 && 14.3 & {\bf39.2} & {\bf52.5} && {\bf29.7} & {\bf46.4} & 48.7 && 26.4 & 55.2 & 66.2 \\
				&&&& Cluster-NMS$_{W+D}$&&  38.5  &  26.0   && {\bf35.2} & 55.3 & {\bf37.4} && 14.3 & {\bf39.2} & {\bf52.5} && {\bf29.7} & {\bf46.4} & {\bf48.8} && {\bf26.5} & {\bf55.3} & {\bf66.4} \\
				\hline
			\end{tabular}
			
			\label{tab:blendsegmentation}
		\end{spacing}
	\end{table*}
\subsection{Object Detection}
For object detection, YOLO v3, SSD and Faster R-CNN are adopted for evaluation.

\subsubsection{YOLO v3 \cite{yolov3}}
First, we evaluate CIoU loss in comparison to MSE loss, IoU loss and GIoU loss on PASCAL VOC 2007 test set \cite{voc}.
YOLO v3 is trained on PASCAL VOC 07+12 (the union of VOC 2007 trainval and VOC 2012 trainval).
The backbone network is Darknet608.
We follow exactly the GDarknet\footnote{\url{https://github.com/generalized-iou/g-darknet}} training protocol released from \cite{giou}.
Original NMS is adopted during inference.
The performance for each loss has been reported in Table~\ref{table:yolo-voc}.
Besides AP metrics based on IoU, we also report the evaluation results using AP metrics based on GIoU.
As shown in Table~\ref{table:yolo-voc}, GIoU as a generalized version of IoU indeed achieves a certain degree of performance improvement.
DIoU loss only includes distance in our pioneer work \cite{diou}, and can improve the performance with gains of 3.29\% AP and 6.02\% AP75 using IoU as evaluation metric.
CIoU loss takes the three important geometric factors into account, which brings an amazing performance gains, i.e., 5.67\% AP and 8.95\% AP75.
%
%
Also in terms of GIoU metric, we can draw the same conclusion, validating the effectiveness of CIoU loss.
%

\begin{table}[!t]\footnotesize
	
	\setlength{\tabcolsep}{8pt}
	\centering
	\caption{Quantitative comparison of YOLOv3~\cite{yolov3} trained using different losses. The results are reported on the test set of PASCAL VOC 2007. }
	\begin{tabular}{c c c c c c}
		\hline
		\raisebox{-1.5ex}{Loss / Evaluation} & \multicolumn{2}{c}{AP} &&\multicolumn{2}{c}{AP75}\\ [0.5ex]
		\cline{2-3}\cline{5-6}
		& IoU & GIoU &&IoU & GIoU \\
		\hline
		MSE   & 46.1 & 45.1  && 48.6 & 46.7 \\
		$\calL_{IoU}$ & 46.6 & 45.8  && 49.8 & 48.8 \\
		$\calL_{GIoU}$ & 47.7 & 46.9 && 52.2 & 51.1 \\
		$\calL_{DIoU}$ & 48.1 & 47.4 && 52.8 & 51.9 \\
		\hline
		$\calL_{CIoU}$ & \textbf{49.2} & \textbf{48.4} && \textbf{54.3} & \textbf{52.9} \\
		\hline
	\end{tabular}
	\label{table:yolo-voc}
	
\end{table}

	\begin{table*}[!t]\footnotesize
		\centering
		\setlength{\tabcolsep}{2.5pt}
		\caption{Comparison of different NMS methods on pre-trained Pytorch-YOLO v3 model. The results are reported on MS COCO 2017 validation set. }
		\newcommand{\modelname}[1]{\methodname{}-#1}
		\begin{spacing}{1.2}
			\begin{tabular}{c c| c| cc cc| ccc c ccc| c ccc c ccc}
				\hline
				Method && NMS Strategy &&    FPS     &     Time   &&  AP  & AP$_{50}$ & AP$_{75}$ &&  AP$_{S}$ &  AP$_{M}$ &  AP$_{L}$ && AR$_1$ & AR$_{10}$ & AR$_{100}$ && AR$_{S}$ & AR$_{M}$ &  AR$_{L}$ \\
				\hline
				\multirow{8}{*}{YOLO v3}
				&&    Fast NMS      && \bf{71.9} & \bf{13.9} &&   42.7   &   63.0   &   45.8   &&   27.9   &    47.8  &    53.2  &&   34.7   &    56.4  &   60.1   &&   44.7   &   64.8   &    73.4   \\
				&&  Original NMS    &&   9.6   &   103.9   &&   43.2   &   63.2   &   46.5   &&   28.3   &    48.4  &    53.7  &&   34.7   &    57.7  &   62.7   &&   46.9   &   67.8   &    75.8   \\
				&&  Original NMS (TorchVision)&&69.0&14.5  &&   43.2   &   63.2   &   46.5   &&   28.3   &    48.4  &    53.7  &&   34.7   &    57.7  &   62.7   &&   46.9   &   67.8   &    75.8   \\
				&& Weighted-NMS        &&   6.2   &  162.3    && 43.6  & \bf{63.4} & 47.4   &&   \bf{29.0}   &    48.9  &    53.7  && \bf{34.9} &    58.1  &   63.0   &&   47.6   &   68.3   &    75.5   \\
				\cline{3-22}
				&&  Cluster-NMS    &&   65.4   &   15.3   &&   43.2   &   63.2   &   46.5   &&   28.3   &    48.4  &    53.7  &&   34.7   &    57.7  &   62.7   &&   46.9   &   67.8   &    75.8   \\
				&& Cluster-NMS$_D$  &&   60.6   &   16.5   &&   43.3   &   63.0   &   47.1   &&   28.4  &  48.6  &  53.8  &&   34.7   &    58.1  &   63.6   &&   47.6   &   68.7   &    76.7   \\
				&& Cluster-NMS$_W$  &&   57.1   &   15.8   && 43.6  & 63.3 & 47.5   &&   \bf{29.0}  &   48.9  &  53.7  && \bf{34.9} &  58.4  &  63.8  &&   48.0   &   69.1   &    76.0  \\
				&& Cluster-NMS$_{W+D}$ &&   53.5   &   17.1   && \bf{43.8}  & 63.0 & \bf{47.9}  &&  \bf{29.0}  &  \bf{49.0}  &  \bf{54.0}  && \bf{34.9} & \bf{58.8} & \bf{64.6} && \bf{48.7} & \bf{69.8} & \bf{77.1} \\
				\hline
			\end{tabular}
			
			\label{tab:yolonms}
		\end{spacing}
	\end{table*}


Since YOLO v3 with GDarknet is implemented using C/C++, it is not suitable for evaluating Cluster-NMS.
Then we evaluate Cluster-NMS on a pre-trained YOLO v3 model in Pytorch, where the model YOLOv3-spp-Ultralytics-608\footnote{\url{https://github.com/ultralytics/yolov3}} achieves much better performance than the original paper of YOLO v3 \cite{yolov3} on MS COCO 2014 validation set.
Table \ref{tab:yolonms} reports the comparison of different NMS methods.
We note that original NMS (TorchVision) is the most efficient due to CUDA implementation and engineering acceleration in TorchVision.
Our Cluster-NMS is only a little slower than original NMS (TorchVision).
By merging coordinates based on overlap areas and scores, Weighted-NMS achieves higher AP and AR than original NMS, but it dramatically lowers the inference speed, making it infeasible for real-time application.
Our Cluster-NMS$_W$ can guarantee the same results with Weighted-NMS, but is much more efficient.
Our Cluster-NMS$_{W+D}$ contributes to further improvements than Cluster-NMS$_W$, especially in terms of AR.
Actually, Cluster-NMS can be further accelerated by logic operations on binarized IoU matrix, but it makes infeasible for incorporating other geometric factors.
Also our Cluster-NMS with geometric factors can still guarantee real-time inference.
%

\subsubsection{SSD \cite{SSD}}
We use another popular one-stage method SSD to further conduct evaluation experiments.
The latest PyTorch implementation of SSD\footnote{\url{https://github.com/JaryHuang/awesome_SSD_FPN_GIoU}} is adopted.
Both the training set and testing set share the same setting with YOLO v3 on PASCAL VOC.
The backbone network is ResNet-50-FPN.
%
And then we train the models using IoU, GIoU, DIoU and CIoU losses.
Table \ref{table:SSD-VOC} gives the quantitative comparison, in which AP metrics based on IoU and evaluation of NMS methods are reported.
As for loss function, we can see the consistent improvements of CIoU loss against IoU, GIoU and DIoU losses.
As for NMS, Cluster-NMS$_{W+D}$ leads to significant improvements in AP metrics, and its efficiency is still well maintained.

\begin{table}[!t]\footnotesize
	
	\setlength{\tabcolsep}{18pt}
	\centering
	
	\caption{Quantitative comparison of SSD~\cite{SSD} for evaluating different loss functions and NMS methods. The results are reported on the test set of PASCAL VOC 2007.}
	
	\begin{tabular}{c|ccc}
		\hline
		{Loss / Evaluation} & AP  & AP$_{75}$ \\
		\hline
		$\calL_{IoU}$& 51.0  & 54.7  \\
		$\calL_{GIoU}$& 51.1  & 55.4  \\
		$\calL_{DIoU}$ & 51.3  & 55.7  \\
		\hline
		$\calL_{CIoU}$ & \textbf{51.5}  & \textbf{56.4}  \\
		\hline
	\end{tabular}\\
	\vspace{0.03in}	
	(a) Comparison of different loss functions.
	\vspace{0.1in}
	
	\setlength{\tabcolsep}{10pt}
	\centering
	\begin{spacing}{1.2}
		\begin{tabular}{c|cc|ccc}
			\hline
			NMS Strategy        &  FPS     & Time  &   AP         &   AP$_{75}$          \\
			\hline
			Fast NMS       & \bf{28.8}     & \bf{34.7} & 50.7           &  56.2         \\
			Original NMS   & 17.8     & 56.1      & 51.5           &  56.4         \\
			\hline
			Cluster-NMS     & 28.0     & 35.7    & 51.5          &  56.4         \\
			Cluster-NMS$_W$     & 26.8     & 37.3     & 51.9           &  56.3          \\
			Cluster-NMS$_{W+D}$ & 26.5     & 37.8     & \bf{52.4}           &  \bf{57.0}         \\
			\hline
		\end{tabular}
	\end{spacing}
	
	(b) Comparison of NMS methods on the model trained by $\mathcal{L}_{CIoU}$.
	\label{table:SSD-VOC}
\end{table}



\subsubsection{Faster R-CNN \cite{fasterrcnn}}
\begin{table}[!htb]
	\footnotesize
	
	\setlength{\tabcolsep}{8pt}
	\centering
	\caption{Quantitative comparison of Faster R-CNN~\cite{fasterrcnn} trained using $\calL_{IoU}$ (baseline), $\calL_{GIoU}$, $\calL_{DIoU}$ and $\calL_{CIoU}$.
		Cluster-NMS$_{W+D}$ is applied on the model trained using $\calL_{CIoU}$, while the other results are produced by Original NMS.
		The results are reported on the validation set of MS COCO 2017.}
	\begin{tabular}{c|cc|ccc}
		\hline
		\multirow{2}{*}{Loss / Evaluation} & \multirow{2}{*}{AP} & \multirow{2}{*}{AP$_{75}$}& \multirow{2}{*}{AP$_S$} & \multirow{2}{*}{AP$_M$} & \multirow{2}{*}{AP$_L$}  \\
		&&&&&  \\
		\hline
		$\calL_{IoU}$& 37.9  & 40.8 & 21.6 & 40.8 & 50.1 \\
		$\calL_{GIoU}$& 38.0 & 41.1 & 21.5 & 41.1 & 50.2 \\
		$\calL_{DIoU}$ & 38.1  & 41.1 & \textbf{21.7} & 41.2 & 50.3 \\
		\hline
		$\calL_{CIoU}$ & 38.7  & 42.0 & 21.3 & 41.9 & 51.5 \\
		Cluster-NMS$_{W+D}$ &\textbf{39.0}  & \textbf{42.3} & \textbf{21.7} & \textbf{42.2} & \textbf{52.1} \\
		\hline
	\end{tabular}
	\label{table:faster-rcnn-cocoval}
\end{table}
We also evaluate CIoU loss for a two-stage object detection method Faster R-CNN\footnote{\url{https://github.com/generalized-iou/Detectron.pytorch}} on MS COCO 2017 \cite{coco}.
Following the same training protocol of \cite{giou}, we have trained the models using CIoU loss in comparison with IoU, GIoU and DIoU losses.
The backbone network is ResNet-50-FPN.
\Tab~\ref{table:faster-rcnn-cocoval} reports the quantitative comparison.
The gains of CIoU loss in AP are not as significant as those in YOLO v3 and SSD.
It is actually reasonable that the initial boxes filtrated by RPN are likely to have overlaps with ground-truth box, and then all DIoU, GIoU and CIoU losses can make good regression.
But due to the complete geometric factors in CIoU loss, the detected bounding boxes will be matched more perfectly, resulting in moderate gains.
%
%
As for suppressing redundant boxes, Cluster-NMS$_{W+D}$ is applied on the model trained by CIoU loss, and it brings further improvements for all the evaluation metrics, validating the effectiveness of assembling geometric factors into Cluster-NMS against original NMS.

\subsection{Discussion}
As for CIoU loss, we have verified the effectiveness of three geometric factors when training deep models of object detection and instance segmentation.
One may notice that CIoU loss yields lower metrics for small objects by Faster-RCNN (Table \ref{table:faster-rcnn-cocoval}), small\&median objects by YOLACT (Table \ref{tab:clusternms1}) and small objects by BlendMask-RT (Table \ref{tab:blendsegmentation}).
The reason may be attributed that for small or median objects, the central point
distance is more important than aspect ratio for regression, and the aspect ratio may weaken the effect of normalized distance between the two boxes.
For large objects, both central point distance and aspect ratio are crucial for bounding box regression, thus resulting in consistent improvements for all these models.
For these models, the average precision and recall on all the objects have validated the superiority of our CIoU loss, although there is some leeway for studying aspect ratio to remedy possible adverse effects for small or median objects, which can be left in future work.

As for Cluster-NMS, from Tables \ref{tab:clusternms1}, \ref{tab:blendsegmentation} and \ref{tab:yolonms}, one can see that these state-of-the-art deep models of object detection and instance segmentation can achieve the highest AP and AR metrics by collaborating with our Cluster-NMS, but they may need different Cluster-NMS versions, e.g., Cluster-NMS$_{W+D}$ for YOLO v3 and BlendMask-RT, and Cluster-NMS$_{S+D}$ for YOLACT.
For a given deep model of object detection or instance segmentation, their detected boxes may have specific property and distribution, thus requiring different NMS methods to obtain the best performance.
However, this is not the drawback of our Cluster-NMS, since Cluster-NMS is a general container to accommodate various NMS methods, while maintaining high inference efficiency.
Meanwhile, the central point distance in Cluster-NMS$_{S+D}$ and Cluster-NMS$_{W+D}$ plays a crucial role, verifying the effectiveness of geometric factors during inference.
%

\section{Conclusion}\label{sec:conclusion}

In this paper, we proposed to enhance geometric factors into CIoU loss and Cluster-NMS for object detection and instance segmentation.
By simultaneously considering the three geometric factors, CIoU loss is better for measuring bounding box regression when training deep models of object detection and instance segmentation.
Cluster-NMS is purely implemented on GPU by implicitly clustering detected boxes, and is much more efficient than original NMS.
Geometric factors can then be easily incorporated into Cluster-NMS, resulting in notable improvements in precision and recall, while maintaining high inference efficiency.
CIoU loss and Cluster-NMS have been applied to the training and inference of state-of-the-art deep object detection and instance segmentation models.
Comprehensive experiments have validated that the proposed methods contribute to consistent improvements of AP and AR, and the high efficiency of Cluster-NMS can guarantee the real-time inference.
CIoU loss and Cluster-NMS can be widely extended to other deep models for object detection and instance segmentation.
%

	{             
		\bibliographystyle{IEEEtran}
		\bibliography{rioucite}
	}
\vspace{-5mm}
\begin{IEEEbiography}[{\includegraphics[width=1in,height=1.25in,clip,keepaspectratio]{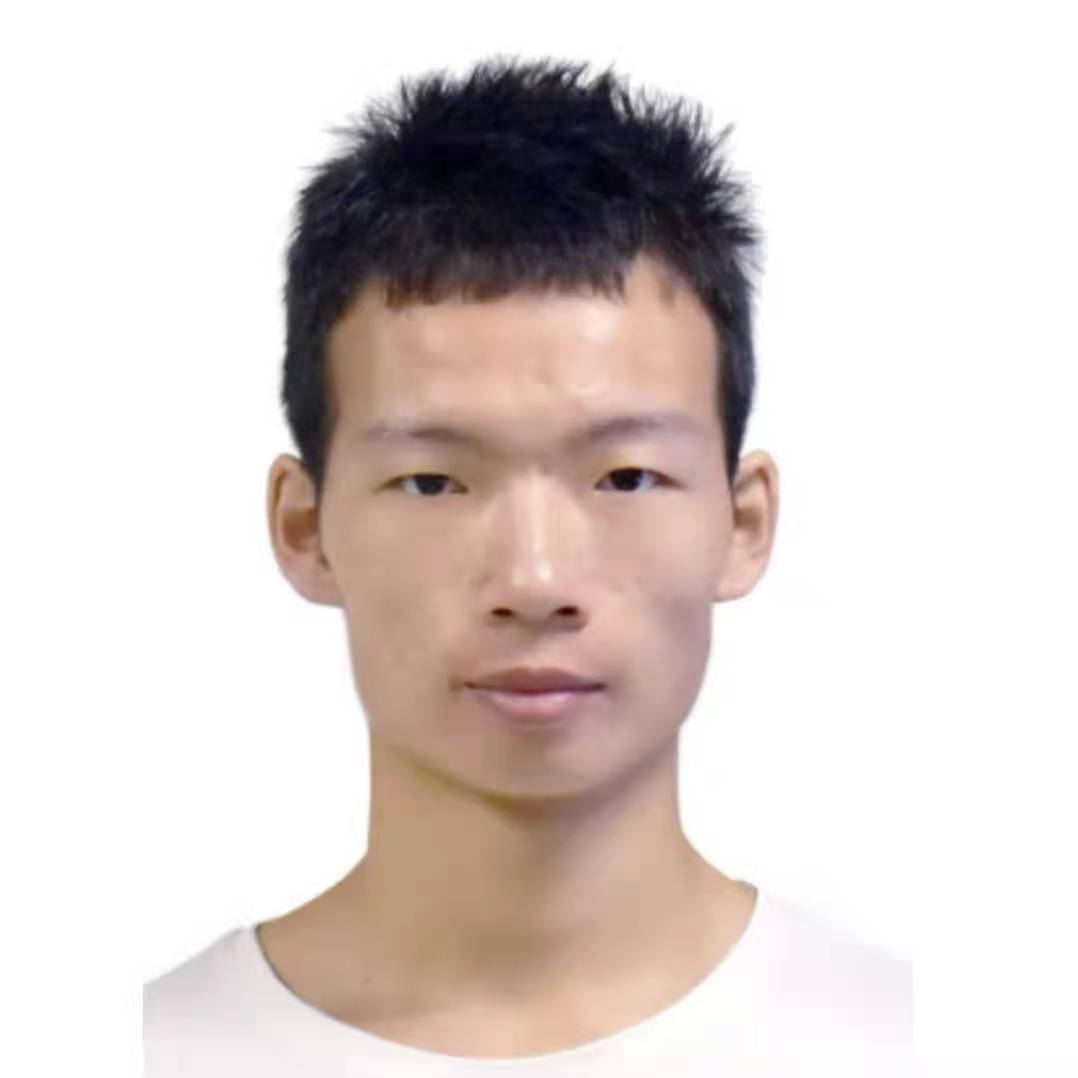}}]{Zhaohui Zheng}
	received the M.S. degree in computational mathematics from Tianjin University in 2021.
	He is currently a Ph.D. candidate with college of computer science at Nankai University, Tianjin, China.
	His research interests include object detection, instance segmentation and scene text detection.
\end{IEEEbiography}
\vspace{-5mm}
\begin{IEEEbiography}[{\includegraphics[width=1in,height=1.25in,clip,keepaspectratio]{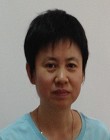}}]{Ping Wang}
    received the B.S., M.S., and Ph.D. degrees in computer science from Tianjin University, Tianjin, China, in 1988, 1991, and 1998, respectively.
    She is currently a Professor with the School of Mathematics, Tianjin University.
    Her research interests include image processing and machine learning.
\end{IEEEbiography}
\vspace{-5mm}
\begin{IEEEbiography}[{\includegraphics[width=1in,height=1.25in,clip,keepaspectratio]{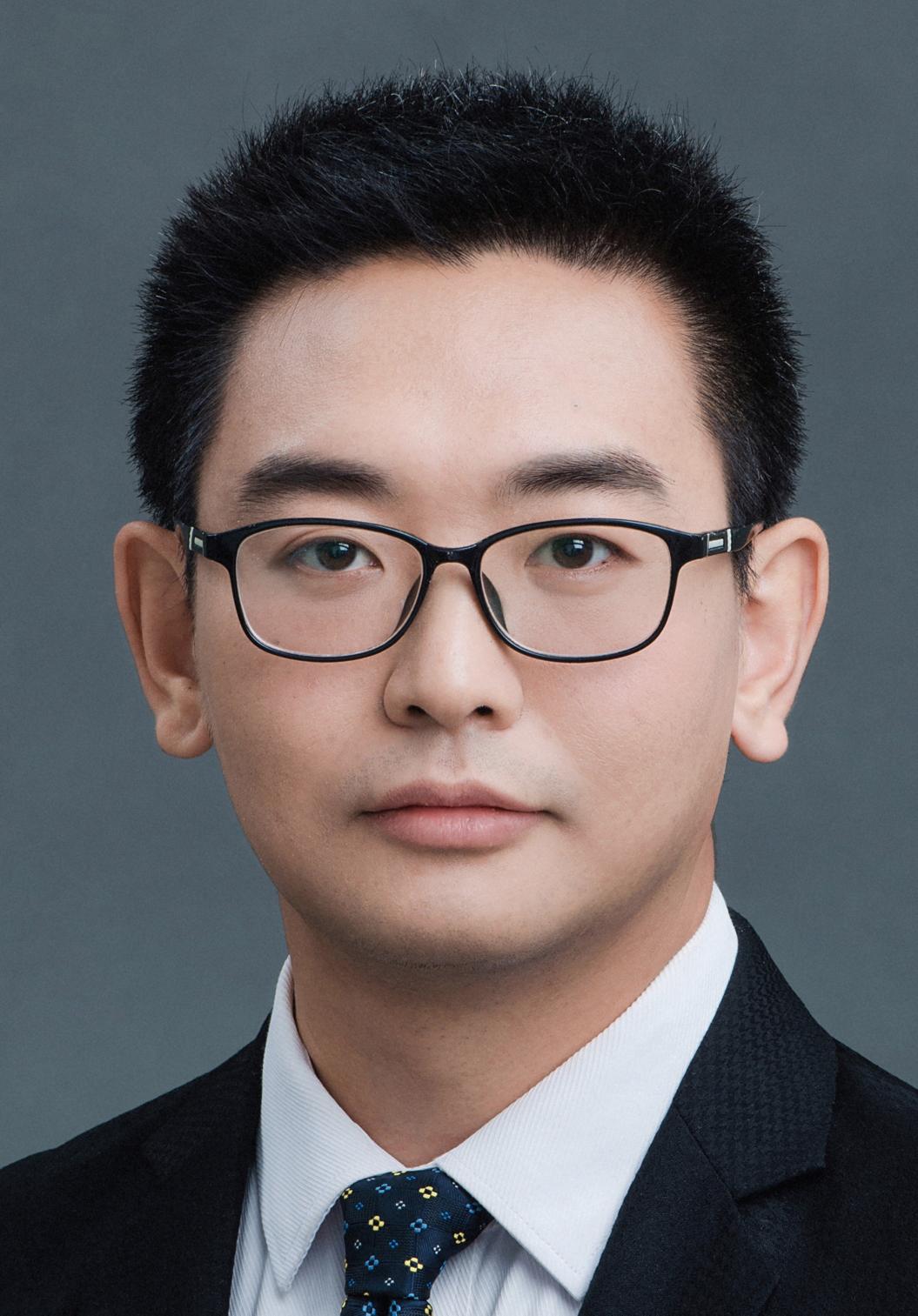}}]{Dongwei Ren}
	received two Ph.D. degrees in computer application technology from Harbin Institute of Technology and The Hong Kong Polytechnic University in 2017 and 2018, respectively.
	From 2018 to 2021, he was an Assistant Professor with the College of Intelligence and Computing, Tianjin University.
	He is currently an Associate Professor with the School of Computer Science and Technology, Harbin Institute of Technology. 
	His research interests include computer vision and deep learning.
\end{IEEEbiography}
\vspace{-5mm}
\begin{IEEEbiography}[{\includegraphics[width=1in,height=1.25in,clip,keepaspectratio]{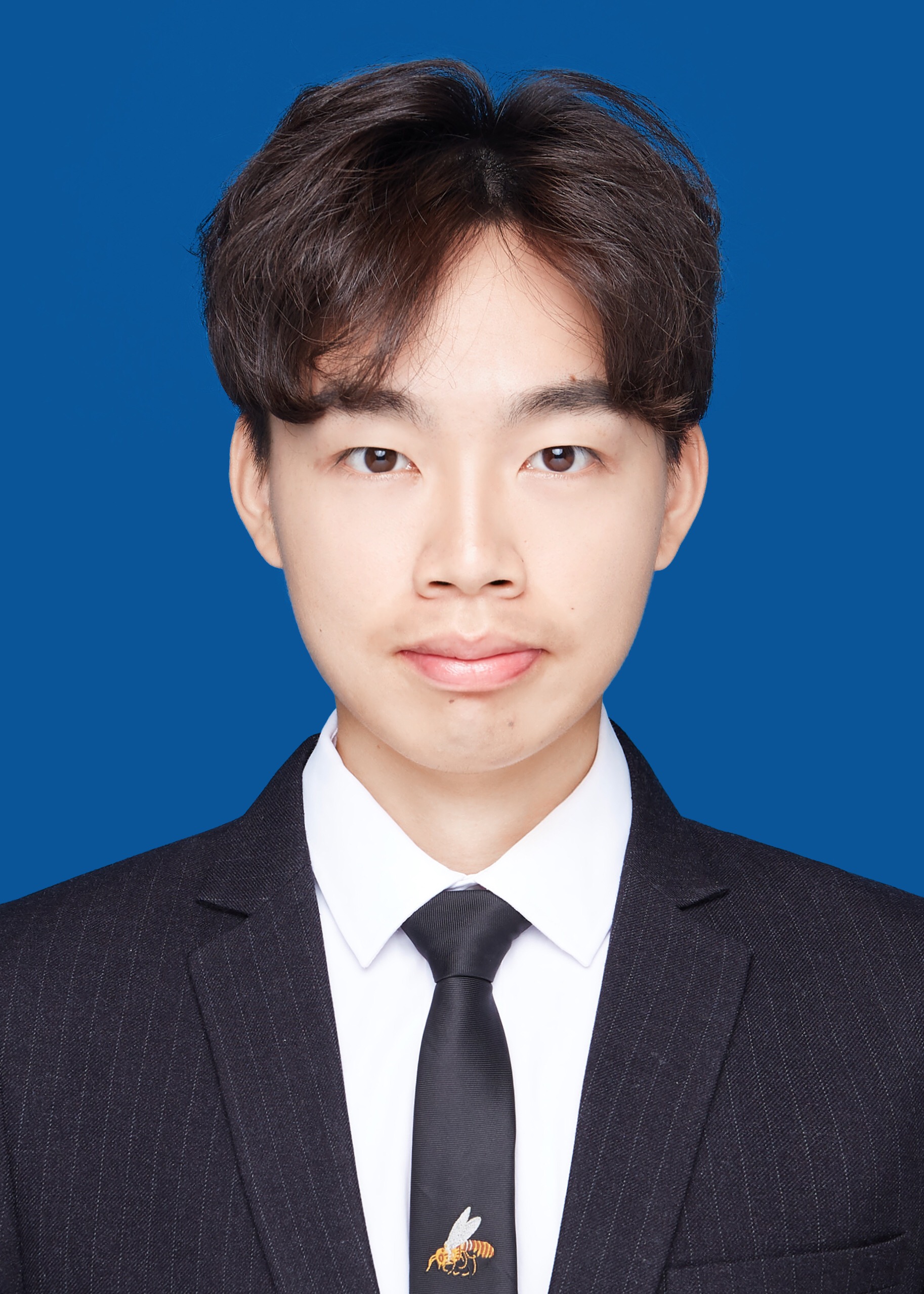}}]{Wei Liu}
    received the B.S and M.S. degrees in computer science from the School of Computer Science and Technology, Tianjin University, China, in 2017 and 2020.
    He is now working at NR Electric Co., Ltd, Nanjing, China.
    His research interests include multimodal computing and computer vision.
\end{IEEEbiography}
\vspace{-5mm}
\begin{IEEEbiography}[{\includegraphics[width=1in,height=1.25in,clip,keepaspectratio]{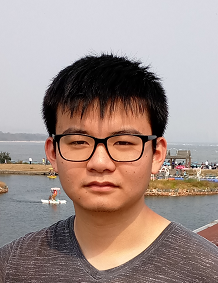}}]{Rongguang Ye}
    received the B.S degree from the School of Mathematics, Tianjin University, Tianjin, China, in 2019.
    He is now pursuing a M.S degree of computational mathematics in Tianjin University.
    His research interests include object detection and computer vision.
\end{IEEEbiography}
\vspace{-5mm}
\begin{IEEEbiography}[{\includegraphics[width=1in,height=1.25in,clip,keepaspectratio]{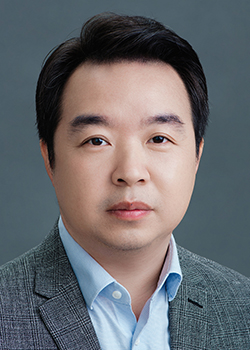}}]{Qinghua Hu}
	received the B.S., M.S., and Ph.D. degrees from the Harbin Institute of Technology, Harbin, China, in 1999, 2002, and 2008, respectively. He was a Post-Doctoral Fellow with the Department of Computing, Hong Kong Polytechnic University, from 2009 to 2011. He is currently the Dean of the School of Artificial Intelligence, the Vice Chairman of the Tianjin Branch of China Computer Federation, the Vice Director of the SIG Granular Computing and Knowledge Discovery, and the Chinese Association of Artificial Intelligence. He is currently supported by the Key Program, National Natural Science Foundation of China. He has published over 200 peer-reviewed papers. His current research is focused on uncertainty modeling in big data, machine learning with multi-modality data, intelligent unmanned systems. He is an Associate Editor of the IEEE TRANSACTIONS ON FUZZY SYSTEMS, Acta Automatica Sinica, and Energies.
\end{IEEEbiography}
\vspace{-5mm}
\begin{IEEEbiography}[{\includegraphics[width=1in,height=1.25in,clip,keepaspectratio]{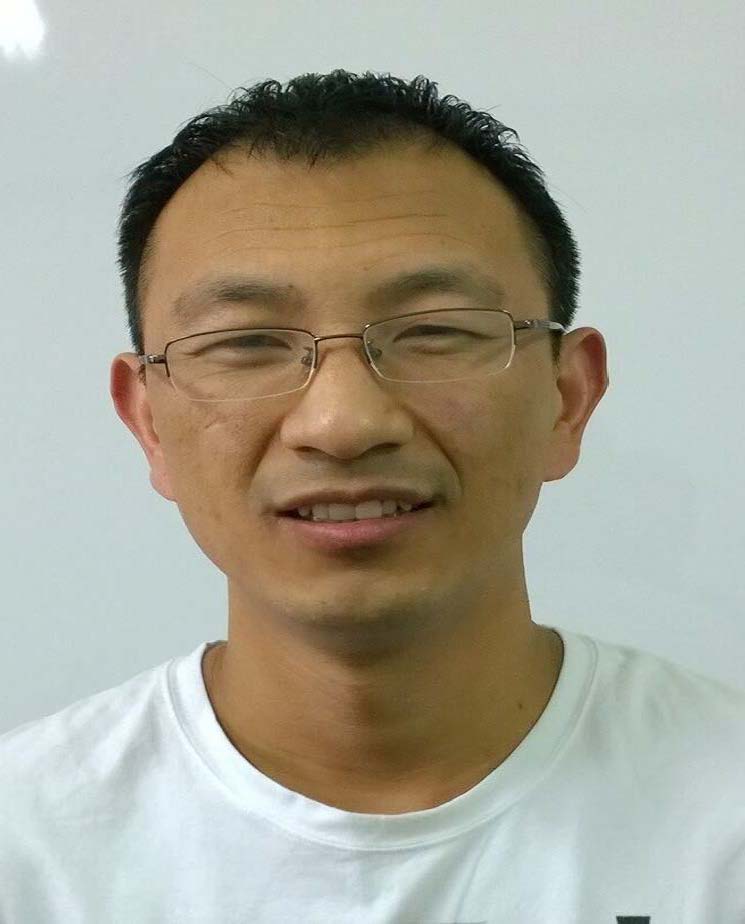}}]{Wangmeng Zuo} (M'09-SM'14)
	received the Ph.D. degree in computer application technology from the Harbin Institute of Technology, Harbin, China, in 2007.
	He is currently a Professor in the School of Computer Science and Technology, Harbin Institute of Technology. His current research interests include image enhancement and restoration, image and face editing, object detection, visual tracking, and image classification. He has published over 100 papers in top tier academic journals and conferences. According to the statistics by Google scholar, his publications have
	been cited more than 20,000 times in literature. He has served as an Associate Editor of the \emph{IEEE Transactions on Pattern Analysis and Machine Intelligence} and \emph{IEEE Transactions on Image Processing}.
	\end{IEEEbiography}

\end{document}